\documentclass{article} 
\usepackage{iclr2025_conference,times}

\iclrarxiv

\usepackage{hyperref}
\usepackage{url}

\usepackage{amsfonts}       
\usepackage{nicefrac}       
\usepackage{microtype}      
\usepackage{xcolor}         
\usepackage{amsmath}
\usepackage{amssymb}
\usepackage{amsthm}
\usepackage{hyperref}
\usepackage[capitalize,noabbrev]{cleveref}
\usepackage{multirow}
\usepackage{multicol}
\usepackage{graphicx}
\usepackage{caption, subcaption}
\usepackage{wrapfig}
\usepackage{makecell}
\usepackage{enumitem}
\usepackage[ruled,linesnumbered,lined,boxed,commentsnumbered]{algorithm2e}
\usepackage{color}
\usepackage{marvosym}
\usepackage{algorithmic}
\usepackage{hyperref}
\usepackage[frozencache,cachedir=.]{minted}
\usepackage{xcolor} 
\usepackage{xspace}

\usepackage{tikz}
\usetikzlibrary{arrows.meta, shapes.geometric, arrows, automata, positioning, backgrounds, fit, calc}

\usepackage{booktabs}

\usepackage{amsmath,amsfonts,amsthm,bm}
\usepackage{xspace}
\usepackage{enumitem}

\def\eqref#1{equation~\ref{#1}}

\def\1{\bm{1}}

\def\rve{{\mathbf{e}}}

\def\rvr{{\mathbf{r}}}

\def\vb{{\bm{b}}}
\def\vc{{\bm{c}}}

\def\vk{{\bm{k}}}

\def\vo{{\bm{o}}}
\def\vp{{\bm{p}}}
\def\vq{{\bm{q}}}

\def\vs{{\bm{s}}}

\def\vu{{\bm{u}}}
\def\vv{{\bm{v}}}

\def\vx{{\bm{x}}}

\def\vz{{\bm{z}}}

\def\mA{{\bm{A}}}

\def\mH{{\bm{H}}}
\def\mI{{\bm{I}}}

\def\mK{{\bm{K}}}

\def\mM{{\bm{M}}}

\def\mP{{\bm{P}}}
\def\mQ{{\bm{Q}}}

\def\mV{{\bm{V}}}
\def\mW{{\bm{W}}}
\def\mX{{\bm{X}}}

\DeclareMathAlphabet{\mathsfit}{\encodingdefault}{\sfdefault}{m}{sl}
\SetMathAlphabet{\mathsfit}{bold}{\encodingdefault}{\sfdefault}{bx}{n}

\def\gB{{\mathcal{B}}}

\def\gE{{\mathcal{E}}}

\def\gL{{\mathcal{L}}}
\def\gM{{\mathcal{M}}}

\def\gS{{\mathcal{S}}}

\def\gV{{\mathcal{V}}}

\def\sN{{\mathbb{N}}}

\def\sR{{\mathbb{R}}}
\def\sS{{\mathbb{S}}}

\newcommand{\ind}{\boldsymbol{1}}
\newcommand{\tok}[1]{\texttt{#1}}
\newcommand{\op}[1]{\(\operatorname{#1}\)}

\newtheorem{theorem}{Theorem}[section]
\newtheorem{proposition}[theorem]{Proposition}
\newtheorem{lemma}[theorem]{Lemma}
\newtheorem{corollary}[theorem]{Corollary}

\theoremstyle{definition}
\newtheorem{definition}[theorem]{Definition}
\newtheorem{assumption}[theorem]{Assumption}

\makeatletter
\@ifundefined{remark}{%
  
  \theoremstyle{remark}
}{%
}
\makeatother

\definecolor{bg}{rgb}{0.95, 0.95, 0.95}

\let\cref\Cref

\crefname{theorem}{theorem}{theorems}
\Crefname{theorem}{Theorem}{Theorems}
\crefname{proposition}{proposition}{propositions}
\Crefname{proposition}{Proposition}{Propositions}
\crefname{lemma}{lemma}{lemmas}
\Crefname{lemma}{Lemma}{Lemmas}
\crefname{corollary}{corollary}{corollaries}
\Crefname{corollary}{Corollary}{Corollaries}
\crefname{fact}{fact}{facts}
\Crefname{fact}{Fact}{Facts}
\crefname{definition}{definition}{definitions}
\Crefname{definition}{Definition}{Definitions}
\crefname{assumption}{assumption}{assumptions}
\Crefname{assumption}{Assumption}{Assumptions}
\crefname{example}{example}{examples}
\Crefname{example}{Example}{Examples}

\def\enf{\ensuremath{E_\textnormal{not-false}}}
\def\eass{\ensuremath{E_\textnormal{assigned}}}
\def\demb{\ensuremath{d_\textnormal{emb}}}
\def\dh{\ensuremath{d_\textnormal{h}}}

\newcommand{\cname}{\texttt{PARAT}\xspace}

\newcommand{\cd}[1]{\texttt{#1}}

\def\attn{\operatorname{Att}}
\def\mha{\operatorname{MHA}}
\def\mlp{\operatorname{MLP}}
\def\emb{\operatorname{Emb}}

\Crefname{algocf}{Algorithm}{Algorithms}

\title{Can Transformers Reason Logically? A Study in SAT Solving}

\author{
  Leyan Pan$^{1}$,
  Vijay Ganesh$^{1,}$\thanks{Equal contribution.},
  Jacob Abernethy$^{1,2,}$\footnotemark[1],
  Chris Esposo$^{1}$,
  Wenke Lee$^{1}$ \\
  $^{1}$Georgia Institute of Technology, Atlanta, GA 30332, USA \quad
  $^{2}$Google Research \\
  \texttt{\{leyanpan, vganesh, prof, cesposo3\}@gatech.edu} \\
  \texttt{wenke@cc.gatech.edu}
}

\begin{document}

\maketitle

\begin{abstract}




We formally study the logical reasoning capabilities of decoder-only Transformers in the context of the boolean satisfiability (SAT) problem. 
First, we prove by construction that decoder-only Transformers can decide 3-SAT, in a non-uniform model of computation, using backtracking and deduction via Chain-of-Thought (CoT).
Second, we implement our construction as a PyTorch model with a tool (\cname) that we designed to empirically demonstrate its correctness and investigate its properties.
Third, rather than \textit{programming} a transformer to reason, we evaluate empirically whether it can be \textit{trained} to do so by learning directly from algorithmic traces (``reasoning paths'') from our theoretical construction. The trained models demonstrate strong out-of-distribution generalization on problem sizes seen during training but has limited length generalization, which is consistent with the implications of our theoretical result.

\end{abstract}
\section{Introduction}

Transformer-based large language models (LLMs,~\citet{transformer}) have demonstrated strong performance on tasks that seem to demand complex reasoning, especially with Chain-of-Thought (CoT, \citet{cot, o1systemcard, deepseekai2025deepseekr1incentivizingreasoningcapability}). However, they often face challenges in reliable multi-step logical reasoning, hallucinating logically flawed or factually incorrect conclusions. Consequently, many researchers reject the idea that LLMs can reason~\cite{llmcantreason}, and researchers continue to disagree on the precise definition of ``reasoning" in the context of LLMs. Furthermore, there is little understanding of the fundamental limitations of the reasoning capabilities of Transformer models.

This paper focuses on the {\it deductive logical} reasoning capability of the Transformer model in a restricted but simple and mathematically precise setting, namely, the Boolean satisfiability problem (SAT,~\citet{cook71}). We view deductive reasoning as the process of systematically drawing valid inferences from existing premises and assumptions. Boolean SAT solving captures the essence of deductive logical reasoning because: 1) Boolean logic lies as the foundation of all logical reasoning, and 2) many modern SAT solvers are inherently formal deductive systems that implement the resolution proof system. Its NP-Completeness also necessitates multiple rounds of trial and error, which is critical for solving complex problems.


We prove by construction that decoder-only Transformers can decide 3-SAT instances with CoT (in a non-uniform computational setting):
\begin{theorem}[Informal version of \cref{thm:sat_search}]
    \label{thm:sat_search_informal}
    For any $p, c \in \sN^+$, there exists a decoder-only Transformer with $O(p^2)$ parameters that can decide all 3-SAT instances of at most $p$ variables and $c$ clauses using Chain-of-Thought.
\end{theorem}

We illustrate the CoT our construction uses to solve 3-SAT instances in \cref{fig:cot}. The Transformer model simulates logical assumption, deduction, and backtracking by generating new tokens and ultimately outputs \cd{SAT}/\cd{UNSAT} as the result of the 3-SAT decision problem. Notably, only a single pass of the model is required to perform logical deduction over {\it all} clauses of the formula based on the current variable assignments (see \cref{lemma:parallel}). Our construction also indicates that softmax attention errors prevent fixed Transformer weights from solving larger 3-SAT instances and limit length generalization from smaller training cases.

\begin{figure*}[h]
\centering
\begin{tikzpicture}[
    node distance=0.5cm and 0.5cm,
    every node/.style={font=\small},
    decision/.style={rectangle, rounded corners, draw=blue!50, fill=blue!20, very thick, minimum height=0.5cm, minimum width=2cm, align=center},
    unit/.style={rectangle, rounded corners, draw=olive!50, fill=olive!20, very thick, minimum height=0.5cm, minimum width=2cm, align=center},
    learned/.style={rectangle, rounded corners, draw=purple!50, fill=purple!20, very thick, minimum height=0.5cm, minimum width=2cm, align=center},
    backtrack/.style={rectangle, rounded corners, draw=red!50, fill=red!20, very thick, minimum height=0.5cm, minimum width=2cm, align=center},
    sat/.style={rectangle, rounded corners, draw=cyan!50, fill=cyan!20, very thick, minimum height=0.5cm, minimum width=2cm, align=center},
    clause/.style={rectangle, draw=black!50, fill=yellow!20, thick, minimum height=0.5cm, minimum width=3.1cm, align=center},
    dimacs/.style={rectangle, draw=black!50, fill=orange!20, thick, minimum height=0.5cm, minimum width=3.1cm, align=center},
    arrow/.style={->, >=stealth', thick},
    dashedarrow/.style={->, >=stealth', dashed, thick},
    deductionarrow/.style={->, very thick, olive!70, opacity=0.3},
    backtrackarrow/.style={->, very thick, red!70, opacity=0.3},
    comment/.style={font=\footnotesize, align=center, text width=3cm},
    ]



\node[clause] (clause1) at (0,0) {\tok{$(\lnot x_2 \lor \lnot x_4 \lor \lnot x_1)$}};
\node (and1) [right=0cm of clause1.east] {\tok{$\land$}};


\node[clause] (clause2) at ($(clause1)+(3.5cm,0)$) {\tok{$(x_3 \lor x_4 \lor \lnot x_1)$}};
\node (and2) [right=0cm of clause2.east] {\tok{$\land$}};


\node[clause] (clause3) at ($(clause1)+(7cm,0)$) {\tok{$(\lnot x_1 \lor \lnot x_3 \lor \lnot x_2)$}};
\node (and3) [right=0cm of clause3.east] {\tok{$\land$}};


\node[clause] (clause4) at ($(clause1)+(0cm,-1.0cm)$) {\tok{$(x_1 \lor \lnot x_2 \lor \lnot x_4)$}};
\node (and4) [right=0cm of clause4.east] {\tok{$\land$}};


\node[clause] (clause5) at ($(clause4)+(3.5cm,0)$) {\tok{$(\lnot x_4 \lor x_2 \lor x_1)$}};
\node (and5) [right=0cm of clause5.east] {\tok{$\land$}};


\node[clause] (clause6) at ($(clause4)+(7cm,0)$) {\tok{$(x_1 \lor \lnot x_2 \lor x_4)$}};

\node[draw=gray, thick, rounded corners, inner sep=0.5cm, fit=(clause1)(clause6)] (clauses_box) {};

\node[draw=gray, thick, rounded corners, fill=cyan!30, anchor=south west] at ([shift={(0.3cm,-0.3cm)}]clauses_box.north west) {Model Input (3-SAT formula)};

\coordinate (lineleft) at ($(clause4.south west)+(-0.5cm,-1.0cm)$);
\coordinate (lineright) at ($(clause6.south east)+(1.5cm,-1.0cm)$);
\draw[thick] (lineleft) -- (lineright);

\node at ($(lineleft)!0.5!(lineright)+(0cm,0.3cm)$) {\textbf{Transformer Chain-of-Thought from \cref{thm:sat_search}}};
\node[decision] (D2) at ($(clause4)+(-1cm,-3cm)$) {\tok{Assume 2}};
\node[decision] (D1) [right=of D2] {\tok{Assume 1}};
\node[unit] (-4) [right=of D1] {\tok{-4}};
\node[unit] (3) [right=of -4] {\tok{3}};
\node[backtrack] (BT1) [right=of 3] {\tok{BackTrack}};

\node[decision] (D2b) [below=0.7cm of D2] {\tok{Assume 2}};
\node[learned] (-1) [right=of D2b] {\tok{-1}};
\node[unit] (-4b) [right=of -1] {\tok{-4}};
\node[backtrack] (BT2) [right=of -4b] {\tok{BackTrack}};
\node[learned] (-2) [below=0.7cm of D2b] {\tok{-2}};
\node[decision] (D3) [right=of -2] {\tok{Assume 3}};
\node[decision] (D4) [right=of D3] {\tok{Assume 4}};
\node[unit] (1) [right=of D4] {\tok{1}};

\node[sat] (SAT) [right=of 1] {\tok{SAT}};

\node[comment, above=0cm of D2, text=blue!50] (D2c) {Assume \( x_2 = T \)};
\node[comment, above=0cm of D1, text=blue!50] (D1c) {Assume \( x_1 = T \)};
\node[comment, above=0cm of -4, text=olive!50] (-4c) {Deduce \( x_4 = F \)};
\node[comment, above=0cm of 3, text=olive!50] (3c) {Deduce \( x_3 = T \)};
\node[comment, above=0cm of BT1, text=red!50] (BT1c) {Conflict};
\node[comment, above=0cm of D2b, text=blue!50] (D2bc) {Keep \( x_2 = T \)};
\node[comment, above=0cm of -1, text=purple!50] (-1c) {Learn \( x_1 = F \)};
\node[comment, above=0cm of -4b, text=olive!50] (-4bc) {Deduce \( x_4 = F \)};
\node[comment, above=0cm of BT2, text=red!50] (BT2c) {Conflict Again};
\node[comment, above=0cm of -2, text=purple!50] (-2c) {Learn \( x_2 = F \)};
\node[comment, above=0cm of D3, text=blue!50] (D3c) {Assume \( x_3 = T \)};
\node[comment, above=0cm of D4, text=blue!50] (D4c) {Assume \( x_4 = T \)};
\node[comment, above=0cm of 1, text=olive!50] (1c) {Deduce \( x_1 = T \)};
\node[comment, above=0cm of SAT, text=cyan!50] (SATc) {Solved!};

\draw[arrow] (D2) -- (D1);
\draw[arrow] (D1) -- (-4);
\draw[arrow] (-4) -- (3);
\draw[arrow] (3) -- (BT1);
\draw[arrow] (D2b) -- (-1);
\draw[arrow] (-1) -- (-4b);
\draw[arrow] (-4b) -- (BT2);
\draw[arrow] (-2) -- (D3);
\draw[arrow] (D3) -- (D4);
\draw[arrow] (D4) -- (1);
\draw[arrow] (1) -- (SAT);

\draw[deductionarrow] (clause1.south) to[out=-90, in=90] (-4c);
\draw[deductionarrow] (clause2.south) to[out=-90, in=90] (3c);
\draw[deductionarrow] (clause4.south) to[out=-90, in=90] (-4bc);
\draw[deductionarrow] (clause5.south) to[out=-90, in=90] (1c);

\draw[backtrackarrow] (clause3.south) to[out=-90, in=90] (BT1c);
\draw[backtrackarrow] (clause6.south) to[out=-90, in=90] (BT2c);

\node[draw=gray, thick, rounded corners, inner sep=0.5cm, fit=(D2c.north west)(SAT.south east)] (cot_box) {};

\node[draw=gray, thick, rounded corners, fill=green!50, anchor=south west] at ([shift={(0.3cm,-0.3cm)}]cot_box.north west) {Model Output in \texttt{typewriter} font};

\end{tikzpicture}
\caption{Visualization of the Chain-of-Thought (CoT) process used by our model to solve an example 3-SAT formula described in \cref{thm:sat_search}. The model autonomously performs trial-and-error reasoning, making multiple attempts and backtracking upon encountering conflicts. Here, $T$ represents \textit{True} and $F$ represents \textit{False}. Tokens in \texttt{typewriter} font denote the CoT generated by the model.}
\label{fig:cot}
\end{figure*}

To empirically verify and investigate our construction, we design a tool (\cname) that instantiates the weights of Transformer models based on NumPy code specifying the desired behavior. With \cname, we implemented the construction as a PyTorch Transformer model and empirically validated its correctness on random 3-SAT instances. 




Additionally, we perform training experiments to demonstrate that Transformers can effectively learn from the deductive reasoning and backtracking process encoded as CoT. We show that trained Transformer models can generalize between SAT instances generated from different distributions within the same number of variables $p$. However, LLMs trained on SAT instances with CoT still struggle to solve instances with an unseen number of variables, demonstrating limitations in learning length-generalizable reasoning. These experimental results support our theoretical predictions.

\noindent{\bf Contributions}
We prove by construction that decoder-only Transformers can solve 3-SAT in a non-uniform model of computation by performing logical deduction and backtracking using Chain-of-Thought (CoT). We show that Transformers can perform logical deduction on all conditions (clauses) in parallel instead of checking each condition sequentially. Nevertheless, the construction requires exponentially many CoT steps in the worst case, as implied by the NP-hardness of SAT, although it requires much fewer steps in most instances.
   
We design \cname, a tool to instantiate Transformer model weights that implement specifications written in NumPy-style code. We empirically demonstrate that the instantiated Transformer corresponding to our theoretical construction can perfectly solve 3-SAT instances. 

Finally, our supporting training experiments suggest that training on CoT encoding 3-SAT reasoning traces allows Transformer models to achieve out-of-distribution generalization within the same input lengths, but fail to generalize to larger instances.

\section{Related Work}

\noindent{\bf Transformers and $\operatorname{P}$ and $\operatorname{P/poly}$ Problems.}  This line of research focuses on what types of computation can Transformer models simulated by providing theoretical constructions of Transformer models that can solve well-defined computational problems.
The seminal work of \citet{automatashortcut} showed that Transformers can simulate semiautomata using a single pass over only a logarithmic number of layers w.r.t. the number of states.
\citet{dyck} demonstrated that transformers can perform parentheses matching of at most $k$ types of parentheses and $D$ appearance of each ($\operatorname{Dyck}_{k,D}$) with $D+1$ layers. 

However, the computation power of one pass of the Transformer model is fundamentally limited \citep{Merrill2021SaturatedTA, tranformerlimit}, and the success of Chain-of-Thought (CoT) reasoning has sparked research on how CoT can improve upon the expressiveness of Transformer models.
\citet{nnTuringComplete} proved that Transformers can emulate the execution of single-tape Turing machines. \citet{loopedtransformers} showed that Transformers can recurrently simulate arbitrary programs written in a one-instruction-set language.
\citet{cotcircuit} proved that Transformers can simulate arbitrary boolean circuits using CoT by representing the circuit in the positional encoding. In particular, transformers can decide all problems in $\operatorname{P/poly}\supseteq \operatorname{P}$ with polynomial steps of CoT.
\citet{cotregular} showed that Transformers with saturated attention can decide all regular languages with a linear number of CoT tokens and decide all problems in P with a polynomial number of CoT tokens. \cite{cottheory} shows that Transformer CoT can perform integer arithmetic, solve linear equations, and perform dynamic programming for the longest increasing subsequence and edit distance problems.


\noindent{\bf How our work differs.} We focus on 3-SAT, which is an NP-complete problem. It is widely believed that  $\operatorname{P}$ is a strict subset of $\operatorname{NP}$, and it is not known whether $\operatorname{NP}$ is a subset of $\operatorname{P/poly}$. In other words, our results are not comparable to these earlier results.

\noindent{\bf Turing Completeness of Transformers.} Meanwhile, \citet{nnTuringComplete}, \citet{cotcircuit}, and \citet{cotregular} also showed that Transformers can simulate single-tape Turing Machines (TM) with CoT and can theoretically be extended to arbitrary decidable languages. However, these constructions require at least one CoT token for every step of TM execution. 

\noindent{\bf How our work differs.} By contrast, our theoretical construction demonstrates that, for certain classes of formal reasoning problems, Transformers can simulate algorithmic reasoning traces at an abstract level with \textit{drastically reduced number of CoT tokens} compared to step-wise emulation of a single-tape TM. At each CoT Step, our construction performs deductive reasoning over the formula in parallel while any single-tape TM must process each input token sequentially. Furthermore, the CoT produced by our theoretical construction abstractly represents the human reasoning process of trial and error, as demonstrated in \cref{fig:cot}.

\textbf{Formal Logical Reasoning with LLMs} Several studies also evaluate pretrained LLMs’ {\it formal} and {\it algorithmic} reasoning abilities, finding that they perform well on a few reasoning steps but struggle as the required steps increase. 
ProofWriter~\citep{tafjord-etal-2021-proofwriter}, ProntoQA~\citep{PrOntoQA, PrOntoQAOOD}, FOLIO~\cite{han-etal-2024-folio}, SimpleLogic~\citep{zhang2022paradox}, and RuleTaker~\citep{clark2020transformers} encodes formal logical reasoning as natural language problems to test general purpose LLMs on multi-step reasoning.
NPHardEval~\cite{fan2023nphardeval} compiles a benchmark of P and NP-Hard problems and tests a variety of pre-trained LLMs. \citet{liu2023code} evaluates code execution capabilities, and \citet{chen2024can} measures capabilities to solve propositional and first-order logic satisfiability as well as SMT formulas.

A related line of work uses formal symbolic logic to enhance the capabilities of LLMs with CoT. LogicLM~\cite{pan-etal-2023-logic} and SymbCoT~\citep{xu-etal-2024-faithful} integrate symbolic expressions of first-order logic with CoT prompting and invoke solvers to provide feedback the LLM reasoner. ~\citet{ryu2024dividetranslatecompositionalfirstorder} uses divide and conquer to improve upon the above works in terms of translation accuracy.
~\citet{jha2024rlsfreinforcementlearningsymbolic} uses symbolic logic solvers to provide reinforcement learning rewards to improve LLM reasoning. 
Beyond LLMs, NeuroSAT~\citep{selsam2018learning},  
MatSAT~\citep{sato2021matsat}, and SATformer~\citep{shi2022satformer} train different neural networks to learn SAT-solving.

\noindent{\bf How our work differs.} Our work focuses on the theoretical capabilities of Transformer models rather than practical pretrained LLMs and can be viewed as building a theoretical foundation for these results.

\noindent{\bf Compilation of Transformer Weights.} Further, prior work on the theoretical construction of Transformer models rarely provides practical implementations. Notably, \citet{loopedtransformers} provides an implementation of their Transformer construction and demonstrates its execution on several programs. However, the model is initialized ``manually" using prolonged sequences of array assignments. \citet{lindner_tracr_2023} released Tracr, which compiles RASP \citep{rasp} programs into decoder-only Transformer models. RASP is a human-readable representation of a subset of operations that Transformers can perform via self-attention and MLP layers. While having related functionalities, our tool has different goals than Tracr and bears multiple practical advantages for implementing complex constructions, which we detail in \cref{sec:vs_tracr}.
\section{Preliminaries}

\label{sec:transformer}

The Transformer architecture \cite{transformer} is a foundational model in deep learning for sequence modeling tasks. In our work, we focus on the autoregressive decoder-only Transformer, which generates sequences by predicting the next token based on previously generated tokens. It is a relatively complex architecture, and here we only give a precise but quite concise description, and we refer the reader \cite{transformer} among many others for additional details. Given an input sequence of tokens $\mathbf{s} = (s_1, s_2, \dots, s_n)\in \gV^n$, where $\gV$ is a \textit{vocabulary}, a Transformer model $M: \gV^* \rightarrow \gV$ maps $\mathbf{s}$ to an output token $s_{n+1} \in \gV$ by composing a sequence of parameterized intermediate operations. These begin with a token embedding layer, following by $L$ \textit{transformer blocks} (\textit{layers}), each block consisting of $H$ \textit{attention heads}, with embedding dimension $d_\text{emb}$, head dimension $d_h$, and MLP hidden dimension $d_\text{mlp}$. Let us now describe each of these maps in detail.

\noindent{\bf Token Embedding and Positional Encoding.}
Each input token $s_i$ is converted into a continuous vector representation $\operatorname{Embed}(s_i) \in \sR^d$ using a fixed embedding map $\emb(\cdot)$. To incorporate positional information, a positional encoding vector $\vp_i \in \sR^d$ is added to each token embedding. The initial input to the first Transformer block is
$$\mX^{(0)} \gets \left( \emb(s_1) + \vp_1, \dots,\ \emb(s_n) + \vp_n \right) \in \sR^{n \times d}.$$
\paragraph{Transformer Blocks.} For $l=1, \ldots, L$, each block $l$ of the transformer processes an embedded sequence $\mX^{(l-1)} \in \sR^{n \times d}$ to produce another embedded sequence $\mX^{(l)} \in \sR^{n \times d}$. Each block consists of a multi-head self-attention (MHA) mechanism and a position-wise feed-forward network (MLP). We have a set of parameter tensors that includes MLP parameters $\mW_1^{(l)} \in \sR^{d_\text{emb} \times d_{\text{mlp}}^*}$, $\vb_1^{(l)} \in \sR^{d_{\text{mlp}}^*}$, $\mW_2^{(l)} \in \sR^{d_{\text{mlp}} \times d}$, and $\vb_2^{(l)} \in \sR^{d}$, self-attention parameters  $\mW_Q^{(l,h)},\ \mW_K^{(l,h)},\ \mW_V^{(l,h)} \in \sR^{d \times d_h}$ for every $h = 1, \ldots, H$, and multi-head projection matrix $\mW_O^{(l)}\in \sR^{(Hd_h) \times d_\text{emb}}$. We will collectively refer to all such parameters at layer $l$ as $\Gamma^{(l)}$, whereas the self-attention parameters for attention head $h$ at layer $l$ will be referred to as $\Gamma^{(l,h)}$. We can now process the embedded sequence $\mX^{(l-1)}$ to obtain $\mX^{(l)}$ in two stages:
\begin{align*}
 \mH^{(l)}  &\gets \mX^{(l-1)} + \mha \left( \mX^{(l-1)} ; \Gamma^{(l)} \right)\\ \mX^{(l)}  &\gets  \mH^{(l)} + \mlp\left( \mH^{(l)} ; \Gamma^{(l)} \right),
\end{align*}
where
\begingroup
\small
\begin{align*}
    &\mha \left( \mX ; \Gamma^{(l)} \right)   
      :=  \left[ \attn(\mX; \Gamma^{(l,1)}); \ldots; \attn(\mX; \Gamma^{(l,H)})\right]  \mW_O^{(l)} \\
    &\attn(\mX; \Gamma^{(l,h)}) 
      :=  \sigma\left( \frac{\mX \mW_Q^{(l,h)}  (\mW_K^{(l,h)}\mX)^\top}{\sqrt{d_h}}  + \mM \right) \mX \mW_V^{(l,h)} \\
    &\mlp \left( \mH ; \Gamma^{(l)}  \right) 
      := \operatorname{act}\left( \mH \mW_1^{(l)} + \vb_1^{(l)} \right) \mW_2^{(l)} + \vb_2^{(l)}. 
\end{align*}
\endgroup

The $n \times n$ matrix $\mM$ is used as a ``mask'' to ensure self-attention is only backward-looking, so we set $\mM[i,j] = -\infty$ for $i \geq j$ and $\mM[i,j] = 0$ otherwise. $\sigma$ represents the softmax operation. We use the $\operatorname{ReGLU}(\cdot):\sR^{2d_\text{mlp}}\rightarrow \sR^{d_\text{mlp}}$ activation function $\operatorname{act}(\cdot)$ at each position. Given input $\vu \in \sR^{n \times 2d_\text{mlp}}$, for each position $i$ we split $\vu_i$ into two halves $\vu_{i,1},\ \vu_{i,2} \in \sR^{d}$ and, using  $\otimes$ denotes element-wise multiplication, we define
\begin{equation}
    \sigma_{\text{ReGLU}}\left( \vu_i \right) = \vu_{i,1} \otimes \operatorname{ReLU}\left( \vu_{i,2} \right). \label{eq:ReGLU}
\end{equation}

\noindent{\bf Output Layer.}
After the final Transformer block, the output representations are projected onto the vocabulary space to obtain a score for each token. We assume that we're using the greedy decoding strategy, where the token with the highest score at the last input position is the model output.
\begin{equation*}
    \vo = \mX^{(L)} \mW_{\text{out}} + \vb_{\text{out}} \in \sR^{n \times V},s_{n+1} = \operatorname{arg}\max_{v}\vo_{n,v} \in \gV \label{eq:token_prob}
\end{equation*}
where $\mW_{\text{out}} \in \sR^{d \times V}$, $\vb_{\text{out}} \in \sR^{V}$, $V$ is the size of the vocabulary, $\vo_{n,v}$ is the score for token $v$ at the last input position $n$.

\noindent{\bf Autoregressive Decoding and Chain-of-Thought.}
During generation, the Transformer model is repeatedly invoked to generate the next token and appended to the input tokens, described in \cref{alg:next-token-prediction}. In this paper, we refer to the full generated sequence of tokens as the \textbf{Chain-of-Thought (CoT)}, and the number of chain-of-thought tokens in \cref{alg:next-token-prediction} is $t-n$.

\begin{algorithm}[ht!]
\caption{Greedy Decoding}
\label{alg:next-token-prediction}
\begin{algorithmic}[1]
\REQUIRE Model $M: \gV^* \rightarrow \gV$, stop tokens $\mathcal{E} \subseteq \gV$, prompt $\vs_{1:n} = (s_1, s_2, \dots, s_n)$,  $t \leftarrow n$
\WHILE{$t \leftarrow t + 1$}
    \STATE $s_t \leftarrow M(\vs_{1:t-1})$ 
    \STATE \textbf{if }{$s_t \in \mathcal{E}$} \textbf{return} $\vs_{1:t}$
\ENDWHILE
\end{algorithmic}
\end{algorithm}


Finally, we refer readers to \cref{sec:DIMACS} and \cite{satbook} for details on SAT solving and 3-SAT.

\section{Transformers and SAT: logical deduction and backtracking}



This section presents and explains our main results on Transformers' capability in deductive reasoning and backtracking with CoT. To rigorously state our results, we first formally define decision problems, decision procedures, and what it means for a model to ``solve" a decision problem using CoT:
\begin{definition}[Decision Problem]
Let $\gV$ be a vocabulary, $\Sigma \subseteq \gV$ be an alphabet, $L \subseteq \Sigma^*$ be a set of valid input strings. We say that a mapping $f: L \rightarrow \{0,1\}$ is a \textit{decision problem} defined on $L$.
\end{definition}

\begin{definition}[Decision Procedure]
 We say that an algorithm $\mathcal{A}$ is a decision procedure for the decision problem $f$, if given any input string $x$ from $L$, $\mathcal{A}$ halts and outputs $1$ if $f(x)=1$, and halts and outputs $0$ if $f(x)=0$.
\end{definition}

\begin{definition}[Autoregressive Decision Procedure]
\label{def:autoregressive-decision}
 For any map $M: \gV^* \rightarrow \gV$, which we refer to as an \textit{auto-regressive next-token prediction model}, and $\mathcal{E}=\{\mathcal{E}_0, \mathcal{E}_1\}\subset \gV$, define procedure $\mathcal{A}_{M,\mathcal{E}}$ as follows: For any input $s_{1:n}$, run Algorithm~\ref{alg:next-token-prediction} with stop tokens $\mathcal{E}$. $\mathcal{A}_{M,\mathcal{E}}$ outputs 0 if $s_{1:t}$ ends with $\mathcal{E}_0$ and $\mathcal{A}_{M,\mathcal{E}}$ output 1 otherwise. We say $M$ \textit{autoregressively decides} decision problem $f$ if there is some $\gE\subset \gV$ for which $\mathcal{A}_{M,\gE}$ decides $f$.
 
\end{definition}

\begin{definition}[$\operatorname{3-SAT}_{p,c}$]
Let $\operatorname{DIMACS}(p, c)$ denote the set of valid DIMACS encodings of 3-SAT instances with at most $p$ variables and $c$ clauses with a prepended $\tok{[BOS]}$ token and an appended $\tok{[SEP]}$ token. Define $\operatorname{3-SAT}_{p,c}: \operatorname{DIMACS}(p, c) \rightarrow \{0,1\}$ as the problem of deciding whether the 3-SAT formula, encoded via $\operatorname{DIMACS}(p, c)$, is satisfiable. 
\end{definition}

With the above definition, we present the formal statement of our main result: 

\begin{theorem}[Decoder-only Transformers can solve SAT]
\label{thm:sat_search}
For any $p, c\in \sN^+$,
there exists a Transformer model $M: \gV^* \rightarrow \gV$ that autoregressively decides $\operatorname{3-SAT}_{p,c}$ in no more than $p\cdot 2^{p+1}$ CoT iterations. $M$ requires $L=7$ layers, $H=5$ heads, $d_{\text{emb}} = O(p)$, and $O(p^2)$  parameters.
\end{theorem}
Remarks on \cref{thm:sat_search}
\vspace{-2pt}
\begin{itemize}
    \setlength{\itemsep}{2pt}  
    \setlength{\parskip}{2pt}
\item Despite the high upper bound on CoT length, it's rarely reached in practice. In \cref{fig:cotlen} we show that the number of CoT tokens is no greater than $8p\cdot 2^{0.08p}$ for most formulas
    \item The worst-case CoT length is independent of the number of clauses $c$, which is due to the parallel deduction over all clauses within the Transformer construction.
    \item Positional encodings are not included in the number of parameters. The positional encoding at position $i$ is the numerical value $i$ at a particular dimension.
    \item Each param. can be represented with $O(p+\log c)$ bits
\end{itemize}

We show our full construction and proof via simulation of the abstract DPLL 
\citep{abstractDPLL} in \cref{sec:sat_search_proof}. The construction uses adapted versions of lemmas from \citet{cottheory} as basic building blocks. Here we provide a proof sketch of the core operations in our theoretical construction.

\noindent{\bf Proof Sketch} 
The Transformer model requires that we encode boolean expressions as vectors. Define the set of literals as the set of variables and their negations $L=\{x_1,\lnot x_1, x_2, \lnot x_2, \dots, x_p,\lnot x_p\}$. Recall that a 3-SAT formula is the conjunction of clauses $C_1\land C_2\land \dots \land C_c$. We view both \textit{partial assignments} $A\subset L$ and \textit{clauses} $C\subset L$ as subsets of literals. For partial assignments, the subset $\{x_1, \lnot x_2, \lnot x_4\}$ denotes the partial assignment $x_1=T,x_2=F,x_4=F$, with $x_3$ unassigned. For clauses, the subset $\{x_1, \lnot x_2, \lnot x_4\}$ denotes the clause $(x_1\lor\lnot x_2 \lor \lnot x_4)$. Note that although we use the same set-based notation for both partial assignments and clauses, they have different meanings: In a partial assignment, each literal specifies the value of a single variable, and all such literals hold simultaneously (an AND) in the partial assignment. In contrast, a clause represents a disjunction (OR) among its literals, meaning at least one must hold true.

We assume without loss of generality that neither $A\subset L$ or $C\subset L$ contain $x_v$ and $\neg x_v$ simultaneously for any $v\in[p]$.  Let $\mathcal{B}\subset P(L)$ all possible subsets of $L$ without the same variable and its negation, then $A\in \mathcal{B}$ and $C\in \mathcal{B}$. We define the vector encoding of partial assignments and clauses as follows:

\begin{figure*}
    \centering
    \includegraphics[width=0.7\linewidth]{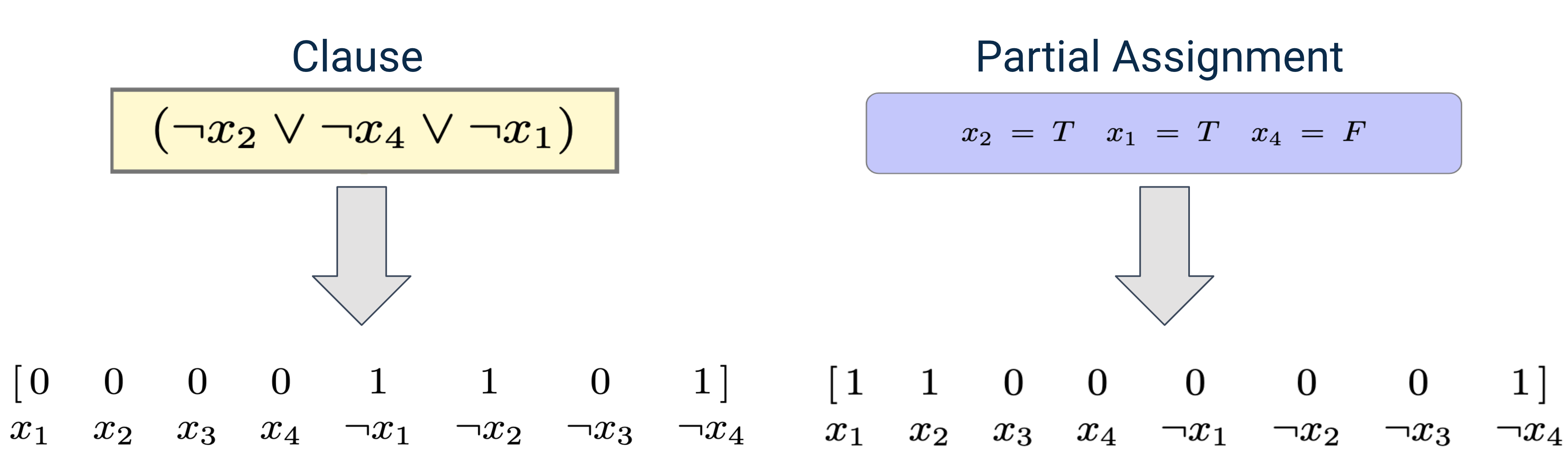}
    \caption{Illustration of the encoding scheme $E(C)$ and $E(A)$ for clauses and partial assignments from \cref{def:clause_encode} with $p=4$.}
    \label{fig:encoding}
\end{figure*}
\begin{definition}[Encoding of clauses and partial assignments, extending \citet{sato2021matsat}]
\label{def:clause_encode}
The mappings $E,E_\textnormal{not-false}, E_\textnormal{assigned} : \mathcal{B} \to \mathbb{R}^{2p}$ encodes $B \in \mathcal{B}$ as  
\begin{align*}
    E(B)_v &:= \boldsymbol{1}_{x_v \in B} 
        &  E(B)_{v+p} &:= \boldsymbol{1}_{(\neg x_v) \in B}.\\
    \enf(B)_v &:= \boldsymbol{1}_{(\lnot x_v) \notin B} 
        &  \enf(B)_{v+p} &:= \boldsymbol{1}_{x_v \notin B}.
\end{align*}
$$    \eass(B)_v :=
          \eass(B)_{v+p} := \boldsymbol{1}_{x_v \in B\lor(\lnot x_v)\in B}.$$

\end{definition}


Given a formula $F$, and a partial assignment of the variables $A$, we may reduce $F$ into a simpler form by evaluating clauses using the assignments in $A$. (See \cref{sec:DIMACS} for details) If $F$ reduces to true under $A$, then $F$ must evaluate to true any full assignment that extends $A$ and we denote $A\models F$.
Conversely, if $F$ reduces to false under $A$, then $F$ must evaluate to false under any full assignment that extends $A$ and we denote $F\models \lnot A$. 
After reducing $F$ under $A$, if there's a clause $C_u$ in $F$ that has only a single literal $l$ left (i.e., $C_u=\{l\}$), then any assignment that extends $A$ and satisfies $F$ must contain $l$. This is the ``deduction" we referred to in \cref{fig:cot} and is formally called "unit propagation", and we denote as $F\land A\models_1 l$.

\setlength{\abovedisplayskip}{5pt}
\setlength{\belowdisplayskip}{5pt}
We now show that the above logical operations can be computed using vector operations on the encoding of the clauses $\{C_1, \dots, C_c\}$ and a partial assignment $A$:
\begin{lemma}
\label{lemma:vec_ded}
Let $F=\bigwedge_{i\in [c]} C_i$ be a 3-SAT formula over p variables $\{x_1, \dots, x_p\}$ and $c$ clauses $\{C_1, \dots, C_c\}$. Let $A\subset L$ be a partial assignment defined on variables $\{x_1, \dots, x_p\}$, then the following properties hold:
\vspace{-8pt}
\begin{enumerate}
    \setlength{\itemsep}{0pt}  
    \setlength{\parskip}{0pt}
    \item Satisfiability Checking: 
    \[
    A\models F\quad\Longleftrightarrow\quad\min_{i\in [c]} E(C_i) \cdot E(A) \geq 1.
    \]
    \item Conflict Detection: 
    $$F\models \lnot A \quad\Longleftrightarrow\quad  \min_{i\in[c]}  E(C_i) \cdot \enf(A) = 0.$$
    \item Deduction: Let $D:=\{l \in L\;|\;F\land A\models_1 l\}$ be the literals deducible from $F$ given $A$ via unit propagation. Then we can write $E(D)$ as
    \[
        \begin{aligned}
        \max\Bigl[& \min\Bigl(\sum_{i \in [c]} 
            E(C_i) \boldsymbol {1}_{\{E(C_i)\cdot \enf(A)=1\}}, 1\Bigr) \\
            &\quad  -  \eass(A),\; 0 \Bigr].
        \end{aligned}
    \] 
    where $\max$ and $\min$ are applied element-wise.
\end{enumerate}
\end{lemma}
\vspace{-0.1in}
Each of the above operations can be approximated by an attention head when given the clause and partial assignment encodings. We capture this idea in the following lemma:

\begin{lemma}[Parallel Processing of Clauses, Informal]
\label{lemma:parallel}
    Let $F$ be a 3-SAT formula over variables $\{x_1, \dots, x_p\}$ with $c$ clauses $\{C_1, \dots, C_c\}$ and $A$ a partial assignment defined on variables $\{x_1, \dots, x_p\}$. Let \[
X_{encoding} =
\begin{bmatrix}
0 & 1 & 1 \\
E(C_1) & 0 & 1 \\
\vdots & \vdots & \vdots \\
E(C_c) & 0 & 1 \\
E(A) & 0 & 1
\end{bmatrix} \in \mathbb{R}^{(c+2)\times (2p+2)}
\] 
Then for any $1 > \epsilon > 0$, given $X$ as input, there exists:
\begin{itemize}
    \setlength{\itemsep}{2pt}  
    \setlength{\parskip}{2pt}
    \item An attention head that outputs $\boldsymbol{1}_{A\models F}$ with approximation error bounded by $\epsilon$
    \item An attention head that outputs $\boldsymbol{1}_{F\models \lnot A}$ with approximation error bounded by $\epsilon$
    \item An attention head followed by an MLP layer that outputs $E(D)$ as defined \cref{lemma:vec_ded} with $\|\cdot\|_\infty$ error bounded by $\epsilon$, unless $F\models \lnot A$
\end{itemize}
All weight values are independent of $F$ and $A$ and are bounded by $O(poly(p, c, \log(1/\epsilon)))$
\end{lemma}

Given the above implementations of logical operations, the high-level overview of our constructions works as follows: 1) Find the previous clause separator (\cd{0}) or backtrack token and compute clause encodings $E(C_i)$ and partial assignment encodings $E(A)$ by summing up the one-hot token embeddings. 2) Compute $\boldsymbol{1}_{A\models F}$, $\boldsymbol{1}_{A\not\models F}$, and $E(D)$ as described in \cref{lemma:parallel}. 3) Determine other conditions such as whether there are assumption variables present in the current assignment, etc., that are required to decide the next action. 4) Determine the output token based on a prioritized list of conditions. e.g., if $A\models F$ output the token \cd{SAT}, else if $F\models \lnot A$ and there are assumptions in $A$ output \cd{BackTrack}, etc.

\section{Implementing the Construction with \cname}

In the previous section, we presented a theoretical construction of a Transformer capable of solving SAT instances. However, it can be difficult to gain insights and fully verify its correctness without experimental interactions with the construction. To help address this, we introduce \cname (short for ParametricTransformer), which instantiates Transformer weights based on high-level specifications written as NumPy code performing array operations.



Both \cname and the specification it accepts are based on Python, and the syntax of the \cname is a restricted subset of Python with the NumPy library. Every variable \texttt{v} in \cname is a 2-D NumPy array of shape \texttt{n} $\times$ \texttt{d\_v}, where \texttt{n} denotes the input number of tokens and \texttt{d\_v} is the dimension of the \cname variable \texttt{v}, which can be different for every variable.

A specification ``program" in \cname is composed of a linear sequence of statements (i.e., no control flow such as loops or branching based on \cname variable values is allowed), where each statement assigns the value of an expression to a variable. Let \texttt{v\_1}, \texttt{v\_2}, $\dots$ denote \cname variable names. Each statement involving \cname variables must be one of the following: \textbf{(1) Binary operations} such as \texttt{v\_1 + v\_2}, \texttt{v\_1 * v\_2}, \texttt{v\_1 - v\_2}; \textbf{(2) Index operations} such as \texttt{v\_1[v\_2, :]} or \texttt{v\_1[:, start:end]}, where $\texttt{start}, \texttt{end}\in [d_\texttt{v\_1}]$; or \textbf{(3) Function calls} from a predefined library of functions that take \cname variables as input.


\cname takes in a specification program as well as variable \texttt{out} of dimension $V$ (size of vocabulary) and outputs a PyTorch \texttt{Module} object that implements a Transformer model as defined in Section 2. The following condition is satisfied: For any possible input sequence of tokens $s$ in the vocabulary of length $n$, the token predicted by the Transformer model is the same as the token corresponding to \texttt{out[-1, :].argmax()} (i.e., the token prediction at the last position) when interpreting the specification using the Python interpreter with the NumPy library. We provide more details on our tool and the supported operations in section \cref{sec:appendix_compiler}.

\subsection{Analysis of the Transformer Construction}
\label{sec:compiled_analysis}

With our tool, we successfully implemented our theoretical construction in \cref{thm:sat_search} using the code in \cref{sec:appendix_code} as a PyTorch model. We will refer to this model as the ``compiled" model for the rest of the section. 
With a concrete implementation of our theoretical construction in PyTorch, we empirically investigate 3 questions (1) Does the compiled model correctly decide SAT instances? (2) How many steps does the model take to solve actual 3-SAT instances? (3) How does error induced by soft attention affect reasoning accuracy?

\noindent{\bf Evaluation Datasets} We evaluate our models on randomly sampled DIMACS encoding of 3-SAT formulas. We focus on SAT formulas with exactly 3 literals in each clause, with the number of clauses $c$ between $4.1p$ and $4.4p$, where $p$ is the number of variables.

It is well-known that the satisfiability of such random 3-SAT formulas highly depends on the clause/variable ratio, where a formula is very likely satisfiable if $c/p\ll 4.26$ and unsatisfiable if $c/p\gg 4.26$ \citep{satphasetransition}. This potentially allows a model to obtain high accuracy just by observing the statistical properties such as the $c/p$ ratio. To address this, we constrain this ratio for all formulas to be near the critical ratio $4.26$. Furthermore, our ``marginal" datasets contain pairs of SAT vs UNSAT formulas that differ from each other by only a single literal. This means that the SAT and UNSAT formulas in the dataset have almost no statistical difference in terms of $c/p$ ratio, variable distribution, etc., ruling out the possibility of obtaining SAT vs UNSAT information solely via statistical properties. We also use 3 different sampling methods to generate formulas of different solving difficulties to evaluate our model:

\begin{itemize}
    \setlength{\itemsep}{1pt}  
    \setlength{\parskip}{1pt}
    \item \textbf{Marginal:} Composed of pairs of formulas that differ by only one token.
    \item \textbf{Random:} Formulas are not paired by differing tokens and each clause is randomly generated.
    \item \textbf{Skewed:} Formulas where polarity and variable sampling are not uniform; For each literal, one polarity is preferred over the other. Some literals are also preferred over others.
\end{itemize}

We generate the above 3 datasets for each variable number $4\leq p \leq 20$, resulting in 51 total datasets of 2000 samples each. Each sample with $p$ variables contains $16.4p$ to $17.6p$ input tokens, which is at least $320$ for $p=20$.

\noindent{\bf Model} Unless otherwise stated, the model we experiment with is compiled from the code in \ref{sec:appendix_code} using \cname with max number of variables $p=20$, max number of clauses $c=88$, and exactness parameter $\beta=20$. The model uses greedy decoding during generation.

\noindent{\bf Accuracy} Our compiled model achieves perfect accuracy on all evaluation datasets described above. This provides empirical justification for our theoretical construction for \cref{thm:sat_search} as well as \cname. This result is included in \cref{fig:lenth_generalization} to compare with trained models.

\noindent{\bf How many steps?}  For all formulas we evaluated, the maximum CoT length is bounded by $8p\cdot 2^{0.08p}$, which is significantly less than the theoretical bound of $p\cdot 2^{(p+1)}$. This indicates that the model can use deduction to reduce the search space significantly. See appendix \cref{fig:cotlen}.

\begin{table*}[ht]
\centering
\label{tab:ood_accuracy}
\begin{tabular}{cl ccc | ccc}
\toprule
& & \multicolumn{3}{c}{$p\in [6,10]$} & \multicolumn{3}{c}{$p\in [11,15]$} \\
\cmidrule(lr){3-5} \cmidrule(lr){6-8}
& & Marginal & Random & Skewed & Marginal & Random & Skewed \\
\midrule
&\textbf{Marginal} & 99.88\%  & 99.99\% & 99.99\%  & 99.82\% & 99.89\% & 99.81\% \\
SAT vs UNSAT &\textbf{Random}   & 99.96\% & 100.00\% & 100.00\% & 99.11\% & 99.75\% & 99.55\% \\
&\textbf{Skewed}   & 99.96\% & 100.00\% & 99.99\%  & 99.41\% & 99.74\% & 99.48\% \\
\midrule
&\textbf{Marginal} & 98.50\%  & 97.33\% & 88.72\%  & 98.66\% & 97.57\% & 86.06\% \\
Full Trace Correct &\textbf{Random}   & 99.40\% & 99.04\% & 93.12\% & 98.56\% & 97.99\% & 91.70\% \\
&\textbf{Skewed}   & 99.38\% & 99.16\% & 97.72\%  & 97.02\% & 95.98\% & 91.51\% \\
\bottomrule
\end{tabular}

\caption{Average accuracies (\%) of SAT/UNSAT prediction and full trace accuracy for models trained and tested on different datasets in the training regime for number of variables $p\in [6,10]$ and $p\in [11,15]$. Columns denote train datasets, and rows denote test datasets. Each accuracy is computed over $10000$ total samples. }
\end{table*}

\section{Can Transformer Learn SAT Solving?}
\begin{figure*}[t!]
    \centering
    \begin{subfigure}[b]{0.49\textwidth}
        \centering
        \includegraphics[width=\textwidth]{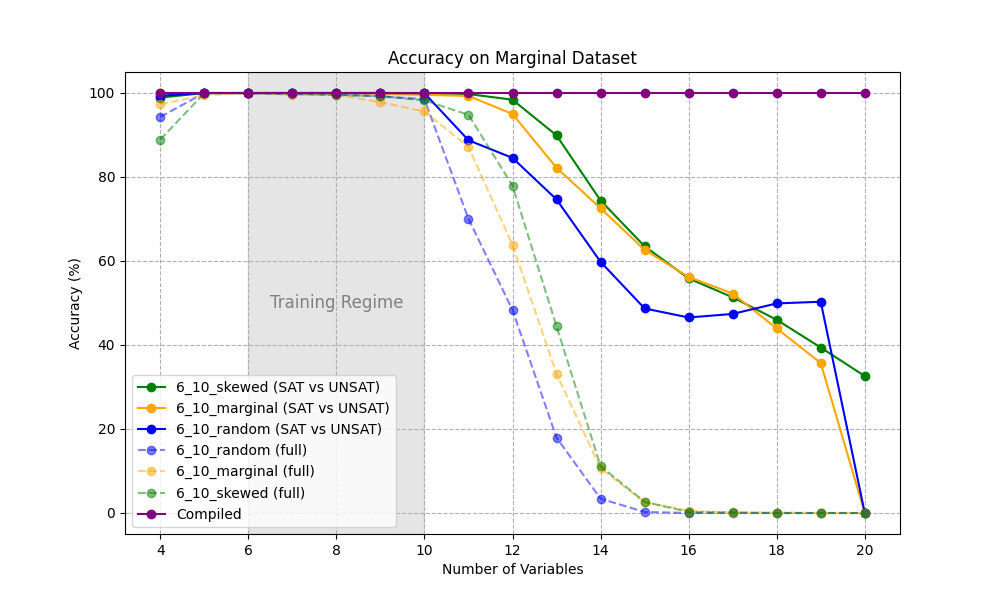}
        \label{fig:subexample1}
    \end{subfigure}
    \hfill
    \begin{subfigure}[b]{0.49\textwidth}
        \centering
        \includegraphics[width=\textwidth]{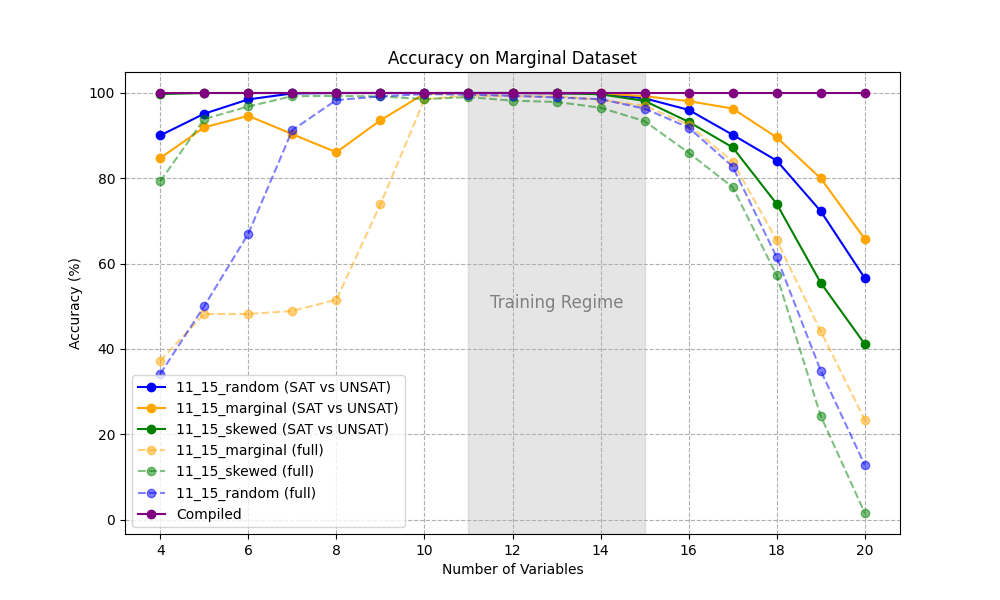}
        \label{fig:subexample2}
    \end{subfigure}
    \vspace{-0.2in}
    \caption{Result of the Length generalization experiments, showing SAT/UNSAT prediction accuracy (solid) and full trace accuracy (opaque, dashed) of Transformer models trained on the marginal, random, and skewed dataset with CoT on the marginal dataset over 4-20 variables. Left: model trained on 6-10 variables. Right: model trained on 11-15 variables. Compiled refers to the compiled model corresponding to our theoretical construction.}
    \label{fig:lenth_generalization}
\end{figure*}

Our previous sections showed that Transformer and weights exist for solving SAT instances using CoT with backtracking and deduction. However, it is unclear to what extent Transformers can learn such formal reasoning procedures by training on SAT formulas. Previously, \citet{paradox} showed that when using a single pass of a Transformer model (without CoT), Transformers fail to generalize to logical puzzles sampled from different distributions even when they have the same number of propositions.

This section provides proof-of-concept evidence that training on the CoT procedure with deduction and backtracking described in \cref{fig:cot} can facilitate Out-of-Distribution generalization within the same number of variables.

\noindent{\bf Datasets} In \cref{sec:compiled_analysis} we introduced 3 different distributions over random 3-SAT formulas of varying difficulties. For training data, we use the same sampling methods, but instead of having a separate dataset for each variable number $p$, we pick 2 ranges $p\in [6,10]$ and $p\in [11, 15]$, where for each sample a random $p$ value is picked uniformly random from the range. Each formula with $p$ variables contains $16.4p$ to $17.6p$ tokens. This results in $2\times 3$ training datasets, each containing $5\times 10^5$ training samples\footnote{The number of training samples is negligible compared to the total number of possible formulas. There are more than $p^{12p}$ 3-SAT formulas with $p$ variables, which is $>10^{56}$ for $p=6$}, with balanced SAT vs UNSAT samples. For each formula, we generate the corresponding CoT in the same format as \cref{fig:cot} using a custom SAT Solver.
The evaluation data is exactly the same as \cref{sec:compiled_analysis}.

\noindent{\bf Model and Training} We use the LLaMa \citep{Touvron2023LLaMAOA} architecture with 70M and 160M parameters for the training experiments, which uses Rotary Positional Encodings (RoPE) and SwiGLU as the activation function for MLP layers. Following prior works \citep{cottheory}, we compute cross-entropy loss on every token in the CoT but not the DIMACS encoding in the prompt tokens. We provide further training details in \cref{sec:appendix_training}. We also permute the variable IDs for training samples to ensure that the model sees all possible input tokens for up to 20 variables.

\noindent{\bf Evaluation Criteria} 
We evaluate our model using two criteria: SAT/UNSAT accuracy and full trace correctness. SAT/UNSAT accuracy evaluates the model's binary prediction based on the first token in $\{\tok{SAT}, \tok{UNSAT}\}$ generated by the model, compared against the ground truth satisfiability of the formula. If the model fails to generate $\{\tok{SAT}, \tok{UNSAT}\}$ within the context length, the prediction is considered incorrect, which can cause accuracy to drop significantly below 50\%. Full trace correctness checks if \emph{every} token generated by the model adheres to the abstract DPLL procedure (\cref{def:abstract_dpll}) under our CoT definition. While strict, the ``correct" CoT is not unique; the model may freely choose variable assignment and deduction orders.

\subsection{Intra-length OOD Generalization}

\label{sec:intra_len}
Our first set of experiments evaluates the model's performance on SAT formulas sampled from different distributions from training, but the number of variables in formulas remains the same ($p\in [6,10]$ and $p\in [11, 15]$ for both train and test datasets).

As shown in \cref{tab:ood_accuracy}, our trained models achieve near-perfect SAT vs UNSAT prediction accuracy when tested on the same number of variables as the training data, even when on formulas sampled from different distributions. The model also strictly follows a correct reasoning procedure for most samples. Recall that the ``marginal" dataset has SAT vs UNSAT samples differing by a single token (out of at least $16p$ tokens in the input formula), which minimizes statistical evidence that can be used for SAT/UNSAT prediction. Our experiments suggest that the LLM have very likely learned logical reasoning procedures using CoT that can be applied to all formulas with the same number of variables as the data they are trained on.

\subsection{Limitations in Length Generalization}
\label{sec:length_generalization}

The second experiment evaluates the model's ability to generalize to formulas with a different number of variables than seen during training. We use the model trained on 3 data distributions described in section \ref{sec:intra_len} and evaluate the marginal dataset with 4-20 variables, generated using the three methods described, with 2,000 samples each. For this experiment, we evaluate the accuracy of the binary SAT vs UNSAT prediction.

\noindent{\bf Results}
In Figure \ref{fig:lenth_generalization}, our results indicate that performance degrades drastically beyond the training regime when the number of variables increases. This shows that the model is unable to learn a general SAT-solving algorithm that works for all inputs of arbitrary lengths, which corroborates our theoretical result where the size of the Transformer for SAT-solving depends on the number of variables. 

\noindent{\bf Ethical Statement} This paper presents work whose goal is to advance the field of Machine Learning. There are many potential societal consequences of our work, none which we feel must be specifically highlighted here.



\bibliography{ref}
\bibliographystyle{iclr2025_conference}

\appendix
\section{Training Details}
\label{sec:appendix_training}
We use Llama \cite{Touvron2023LLaMAOA} models in the HuggingFace library. For the 70M model, we use models with 6 layers, 512 embedding dimensions, 8 heads, 512 attention hidden dimensions, and 2048 MLP hidden dimensions. For the 140M model, we use 12 layers, 768 embedding dimensions, 12 heads, 768 attention hidden dimensions, and 3072 MLP hidden dimensions. Both models have 850 context size. We trained for 5 epochs on both datasets using the Adam optimizer with a scheduled cosine learning rate decaying from $6\times 10^{-4}$ to $6\times 10^{-5}$ with $\beta_1=0.9$ and $\beta_2=0.95$.

\section{Additional Experiment Results}
\paragraph{Number of CoT Tokens for Theoretical Construction}
\paragraph{Effect of Soft Attention}
\paragraph{Length Generalization Results on Additional Datasets}
In \cref{fig:cotlen} we provide results on the number of CoT tokens required to solve randomly generated SAT instances. In \cref{fig:soft_attn} we provide results on how the SAT/UNSAT prediction accuracy is affected by numerical errors introduced by softmax. In \cref{fig:lenth_generalization_appendix} we present results for length generalization (described in \cref{sec:length_generalization}) on the marginal and skewed datasets.

\begin{figure}
    \centering
    \includegraphics[width=0.8\linewidth]{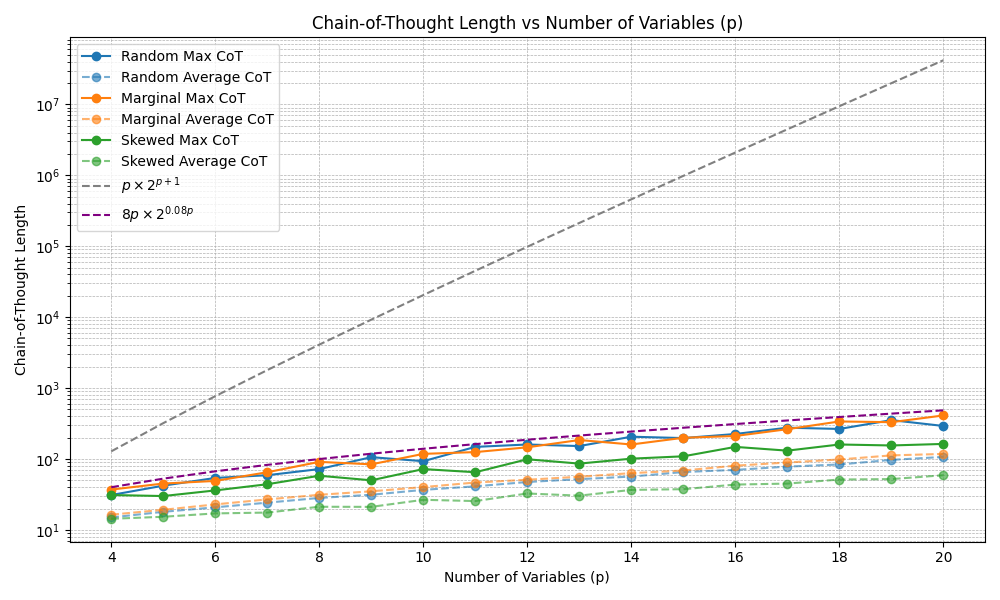}
    \caption{CoT Lengths generated by the compiled SAT-Solver Model vs the number of boolean variables in sampled SAT formulas, y-axis in log scale. Solid lines denote the maximum CoT length for each dataset while opaque, dashed lines denote the average CoT length. The empirical maximum CoT length in our datasets is bounded by $8p\cdot 2^{0.08p}$}.
    \label{fig:cotlen}
\end{figure}

\begin{figure}
    \centering
    \includegraphics[width=0.7\linewidth]{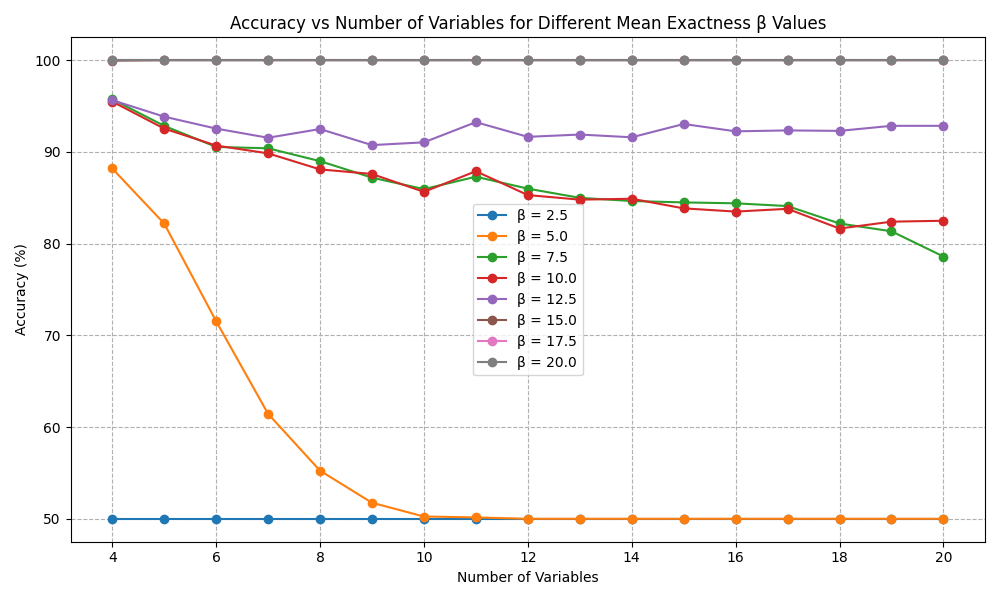}
    \caption{The impact of soft attention in Transformer layers on the SAT/UNSAT prediction accuracy. $\beta$ is a scaling factor that allows the soft attention operation to better simulate hard attention at the cost of larger model parameter values in attention layers. The model achieves perfect accuracy on all ``marginal" datasets starting at $\beta=17.5$, and for lower $\beta$ values, accuracy is negatively correlated with the number of variables in the datasets. }
    \label{fig:soft_attn}
\end{figure}

\begin{figure*}[t!]
    \centering
    \begin{subfigure}[b]{0.49\textwidth}
        \centering
        \includegraphics[width=\textwidth]{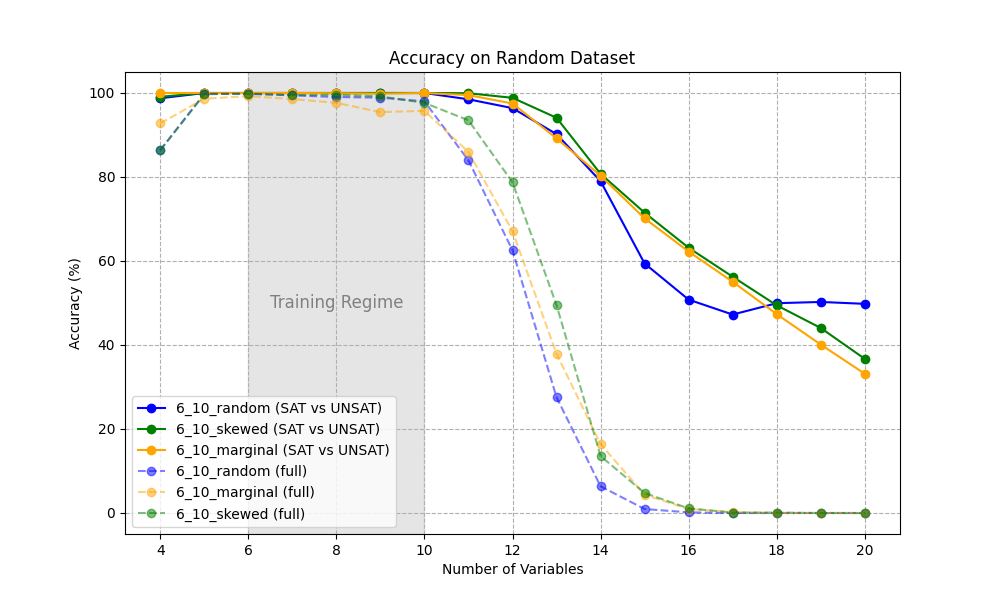}
        \label{fig:trained_6_10_random}
    \end{subfigure}
    \hfill
    \begin{subfigure}[b]{0.49\textwidth}
        \centering
        \includegraphics[width=\textwidth]{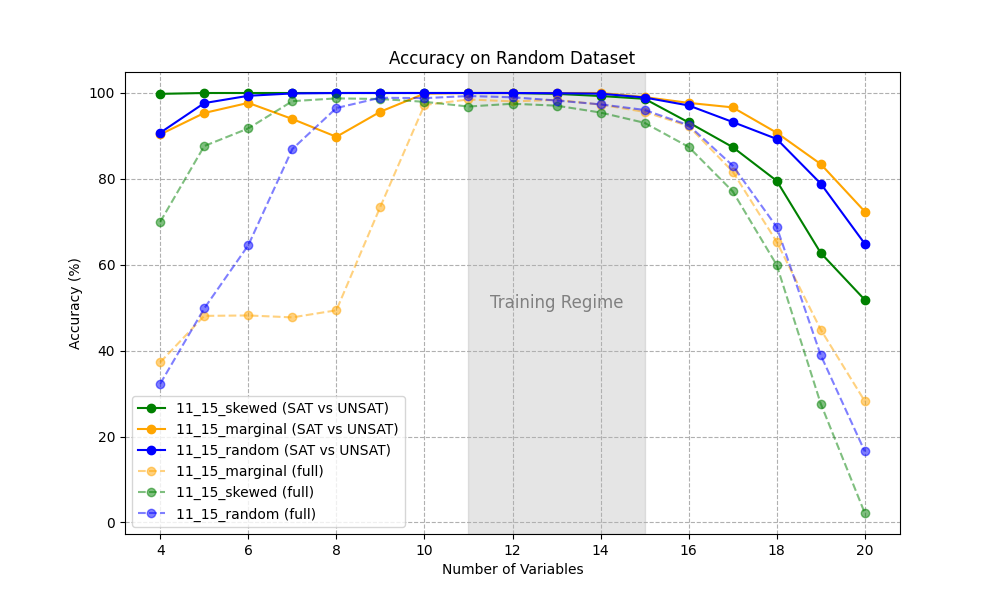}
        \label{fig:trained_11_15_random}
    \end{subfigure}
    \begin{subfigure}[b]{0.49\textwidth}
        \centering
        \includegraphics[width=\textwidth]{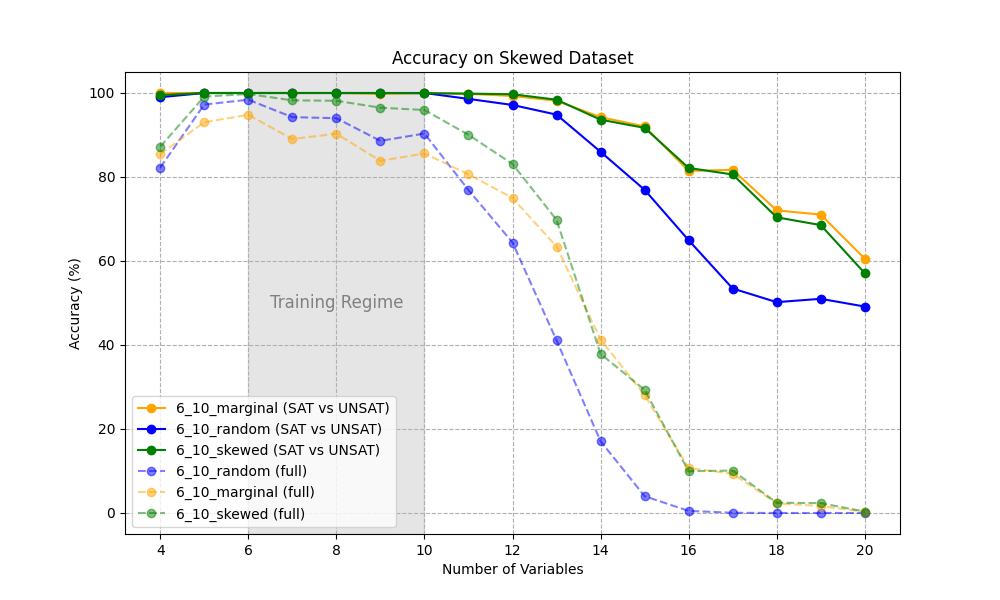}
        \label{trained_6_10_skewed}
    \end{subfigure}
    \hfill
    \begin{subfigure}[b]{0.49\textwidth}
        \centering
        \includegraphics[width=\textwidth]{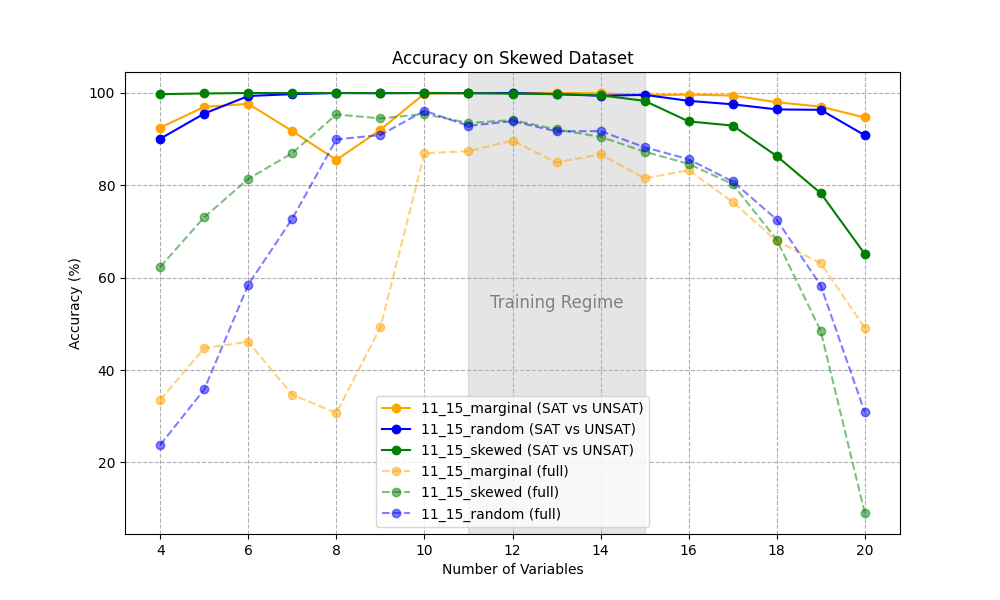}
        \label{fig:trained_11_15_skewed}
    \end{subfigure}
    \caption{Result of the Length generalization experiments on the random and skewed evaluation dataset. The meaning of different lines are the same as \cref{fig:lenth_generalization}}
    \label{fig:lenth_generalization_appendix}
\end{figure*}

\section{Proofs}
\label{sec:sat_search_proof}
\subsection{Preliminaries on SAT Solving}
\label{sec:DIMACS}
\paragraph{SAT} The Boolean satisfiability problem (SAT) is the problem of determining whether there exists an assignment $A$ of the variables in a Boolean formula $F$ such that $F$ is true under $A$. 

\paragraph{3-SAT}
In this paper, we only consider 3-SAT instances in \textit{conjunctive normal form} (CNF), where groups of at most 3 variables and their negations (\textit{literals}) can be joined by OR operators into clauses, and these clauses can then be joined by AND operators. We use the well-known \textit{DIMACS} encoding for CNF formulas where each literal is converted to a positive or negative integer corresponding to its index, and clauses are separated by a 0 (which represents an $\land$ operation). SAT problems where the Boolean formula is expressed in conjunctive normal form (CNF) with three literals per clause will be referred to as \textit{3-SAT}. A formula in CNF is a conjunction (i.e. ``AND") of clauses, a \textbf{clause} is a disjunction (i.e. ``OR") of several \textbf{literals}, and each literal is either a variable or its negation. In the case of 3-SAT, each clause contains at most three literals. An example 3-SAT formula with 4 variables and 6 clauses is:
\[
\begin{aligned}
& (x_1 \lor \neg x_2) \land (\neg x_1 \lor x_2 \lor \neg x_3) \land (x_2 \lor x_4 \lor \neg x_1) \land \\
& (x_1 \lor \neg x_3 \lor x_4) \land (\neg x_2 \lor \neg x_3 \lor \neg x_4) \land (\neg x_4 \lor \neg x_1)\\
\end{aligned}
\]

In the above formula, $(x_1 \lor \neg x_2)$ is a clause, which contains the literals $x_1$ and $\lnot x_2$.

The 3-SAT problem refers to determining if any assignment of truth values to the variables allows the formula $\phi$ to evaluate as true. It is well-known that 3-SAT is NP-hard and is widely believed to be unsolvable in polynomial time.

\paragraph{DIMACS Encoding}
The DIMACS format is a standardized encoding scheme for representing Boolean formulas in conjunctive normal form (CNF) for SAT problems. Each clause in the formula is represented as a sequence of integers followed by a terminating ``0" (i.e. ``0" represents $\land$ symbols and parentheses). Positive integers correspond to variables, while negative integers represent the negations of variables. For instance, if a clause includes the literals $x_1$, $\neg x_2$, and $x_3$, it would be represented as "\cd{1 -2 3 0}" in the DIMACS format.

For the 3-SAT example in the previous paragraph, the corresponding DIMACS representation would be:
\begin{center}
\texttt{1 -2 0 -1 2 -3 0 2 4 -1 0 1 -3 4 0 -2 -3 -4 0 -4 -1 0
}
\end{center}

\paragraph{Reducing a Formula.}
Let 
\[
F \;=\;\bigwedge_{i=1}^c C_i
\]
be a 3-SAT formula, where each $C_i$ is a clause (i.e.\ a disjunction of up to three literals).  The \emph{reduction} of $F$ by $A$, denoted $F|_A$, is defined by:
\begin{enumerate}
    \item \textbf{Remove (drop) any clause satisfied by $A$.}\\ 
    A clause $C_i$ is satisfied by $A$ if there is a literal $\ell \in C_i$ such that $\ell \in A$.  In that case, $C_i$ is automatically \texttt{True} and can be omitted from the conjunction.
    
    \item \textbf{Delete (false) literals contradicting $A$.}\\ 
    For each remaining clause $C_i$, if it contains a literal $\ell$ that is \emph{false} under $A$, remove that literal from $C_i$. Specifically:
    \begin{itemize}
      \item If $x_j \in A$ (so $x_j$ is \texttt{True}), then any literal $\lnot x_j$ in $C_i$ becomes false and is removed.
      \item If $\lnot x_j \in A$ (so $x_j$ is \texttt{False}), then any literal $x_j$ in $C_i$ is removed.
    \end{itemize}
    If a clause loses all its literals through this process, it becomes an \emph{empty clause} and the formula is immediately \texttt{False}.
\end{enumerate}
Formally, for each clause $C_i \subseteq L$, define
\[
C_i|_A 
\;:=\; 
\bigl(C_i \setminus \{\ell \in C_i : \ell \text{ is forced false by }A\}\bigr)
\]
and keep $C_i|_A$ only if it is not already satisfied by $A$. Then
\[
F|_A 
\;=\; 
\bigwedge_{\substack{i=1\\C_i \text{ not satisfied}}}^{c} \bigl(C_i|_A \bigr).
\]

As an example, suppose 
    \[
    F \;=\; (x_1 \lor \lnot x_2) \;\land\; (\lnot x_1 \lor x_3)\;\land\;(x_2 \lor \lnot x_3).
    \]
    Let $A=\{x_1\}$. Then:
    \begin{enumerate}
        \item The first clause $(x_1 \lor \lnot x_2)$ is satisfied by $x_1\in A$.  Hence we \emph{remove} it from the formula.
        \item In the second clause $(\lnot x_1 \lor x_3)$, the literal $\lnot x_1$ is false (since $x_1$ is set \texttt{True}).  We remove $\lnot x_1$ and are left with $(x_3)$.
        \item The third clause $(x_2 \lor \lnot x_3)$ is untouched: $x_1$ does not appear, so no literal is removed. However, it is not satisfied by $x_1$, so we keep it.
    \end{enumerate}
    Thus,
    \[
    F|_A \;=\; (x_3) \;\land\; (x_2 \lor \lnot x_3).
    \]
     If a partial assignment forces a clause to become empty, the whole formula becomes unsatisfiable under that assignment. For instance, with
    \[
    F \;=\;(x_1 \lor x_2) \;\land\;(\lnot x_1 \lor \lnot x_2),
    \]
    and a partial assignment $A=\{x_1, x_2\}$, we see:
    \begin{itemize}
      \item The first clause $(x_1 \lor x_2)$ is satisfied by $x_1\in A$ and gets removed.
      \item In the second clause $(\lnot x_1 \lor \lnot x_2)$, both $\lnot x_1$ and $\lnot x_2$ contradict $A$, so both are removed. This leaves the second clause empty, which means $F|_A$ is an empty conjunction (i.e.\ \texttt{False}).
    \end{itemize}
    Hence no full extension of $A$ can satisfy $F$.

\paragraph{Unit Propagation.}
An additional reduction step performed in SAT solving is \textbf{unit propagation}.  After applying a partial assignment $A$ to a formula $F$ (obtaining $F|_A$), some clauses may reduce to a \emph{single literal} (called a \emph{unit clause}).  Formally, a clause $C = \{\,\ell_1,\dots,\ell_k\}$ is \textbf{unit} if $k=1$.  If $C$ is unit, its lone literal $\ell$ must be assigned \texttt{True} in any extension of $A$ that satisfies $F$.  Concretely:
\begin{enumerate}
    \item \textbf{Identify unit clauses.} Scan the reduced formula $F|_A$.  If there is a clause $C_u$ with exactly one remaining literal $\ell$, then $\ell$ is forced \texttt{True}.
    \item \textbf{Extend the partial assignment.} Insert the forced literal $\ell$ into $A$.  
    \item \textbf{Re-reduce the formula.} Remove any clauses satisfied by $\ell$, and remove $\lnot \ell$ from all remaining clauses.
\end{enumerate}
This process may uncover additional unit clauses in subsequent steps, so unit propagation continues iteratively until there are no more clauses of size 1.  If at any point a clause becomes empty, we conclude that the current assignment $A$ cannot be extended to a satisfying assignment.  

\paragraph{Example.} 
Consider $F = (x_1 \lor \lnot x_2) \land (\lnot x_1 \lor x_3) \land (x_2 \lor \lnot x_3)$ and a partial assignment $A = \{\lnot x_1\}$.  
\begin{itemize}
    \item First, $F|_{A}$ removes $\lnot x_1$ (now satisfied) from $(\lnot x_1 \lor x_3)$, leaving the unit clause $(x_3)$. Thus $x_3$ is forced \texttt{True}.
    \item We add $x_3$ to $A$, giving $A \leftarrow A \cup \{x_3\}$. Re-reducing the formula removes any literal $\lnot x_3$.  If that step causes another clause to become unit, we repeat.
\end{itemize}
This iterative assignment of forced literals often simplifies the problem significantly before any broader search is required.

\subsection{Proof of \cref{lemma:vec_ded}}
We prove each of the three statements in the lemma, showing that the vector-based definitions correspond to the logical operations described.

\subsection*{1. Satisfiability Checking}

\[
A \models F \quad \Longleftrightarrow \quad \min_{i \in [c]} \Big(E(C_i) \cdot E(A)\Big) \geq 1.
\]

\paragraph{Logical Interpretation.}
The left-hand side, $A \models F$, means that every clause $C_i$ in $F$ is satisfied by $A$. This is equivalent to saying that, for every clause $C_i$, there exists at least one literal $l \in C_i$ such that $l \in A$. 

\paragraph{Vector Translation.}
For a clause $C_i$ and a partial assignment $A$, the dot product $E(C_i) \cdot E(A)$ computes the number of literals in $C_i$ that are also in $A$:
\[
E(C_i) \cdot E(A) = \sum_{v=1}^p \ind_{\{x_v \in C_i\}} \cdot \ind_{\{x_v \in A\}} + \sum_{v=1}^p \ind_{\{\lnot x_v \in C_i\}} \cdot \ind_{\{\lnot x_v \in A\}}=|C_i\cap A|.
\]
If $E(C_i) \cdot E(A) \geq 1$, this means there is at least one literal in $C_i \cap A$, and hence $C_i$ is satisfied. Taking the minimum over all clauses ensures that every clause $C_i$ is satisfied, which is precisely the condition for $A \models F$.

\subsection*{2. Conflict Detection}

\[
F \models \lnot A \quad \Longleftrightarrow \quad \min_{i \in [c]} \Big(E(C_i) \cdot \enf(A)\Big) = 0.
\]

\paragraph{Logical Interpretation.}
The left-hand side, $F \models \lnot A$, means that $F$ contradicts $A$, i.e., there exists a clause $C_i$ in $F$ such that all literals in $C_i$ are forced false by $A$. This happens if and only if no literal in $C_i$ is ``not-false'' under $A$.

\paragraph{Vector Translation.}
For a clause $C_i$, the dot product $E(C_i) \cdot \enf(A)$ computes the number of literals in $C_i$ that are \emph{not forced false} by $A$:
\[
E(C_i) \cdot \enf(A) = \sum_{v=1}^p \ind_{\{x_v \in C_i\}} \cdot \ind_{\{\lnot x_v \notin A\}} + \sum_{v=1}^p \ind_{\{\lnot x_v \in C_i\}} \cdot \ind_{\{x_v \notin A\}}.
\]
If $E(C_i) \cdot \enf(A) = 0$, this means all literals in $C_i$ are forced false by $A$, and $C_i$ is a contradiction. Taking the minimum over all clauses ensures that this happens for at least one clause, which corresponds to $F \models \lnot A$.

\subsection*{3. Deduction (Unit Propagation)}

\[
E(D) \;=\; \max\Bigl(
        \min\Bigl(\sum_{i \in [c]} 
            \boldsymbol{1}_{\{E(C_i)\cdot \enf(A)=1\}}
            \cdot E(C_i),\; 1\Bigr) -\; \eass(A),\; 0
        \Bigr).
\]

\paragraph{Logical Interpretation.}
A clause $C_i$ becomes a \emph{unit clause} under $A$ if all but one of its literals are forced false by $A$. In this case, the remaining literal must be set to \texttt{True} in any extension of $A$. The set $D$ consists of all such literals deduced via unit propagation.

\paragraph{Vector Translation.}
For each clause $C_i$, the condition $E(C_i) \cdot \enf(A) = 1$ identifies unit clauses after reduction, i.e., those with exactly one literal not forced false by $A$. For such clauses, $E(C_i)$ encodes the remaining literal.

The summation 
\[
\sum_{i \in [c]} \ind_{\{E(C_i) \cdot \enf(A) = 1\}} \cdot E(C_i)
\]
computes a vector where each coordinate accumulates contributions from unit clauses identifying the corresponding literal. Taking $\min(\cdot, 1)$ elementwise ensures that each coordinate is at most $1$, avoiding overcounting. Finally, subtracting $\eass(A)$ removes literals that are already assigned by $A$, leaving only the newly deduced literals.

This matches the standard logical definition of unit propagation.

\qed
\subsection{Useful Lemmas for Transformers}
In this section, several useful results on Transformer operations on their approximation capavilities. Specifically, an MLP with ReGLU can exactly simulate ReLU, linear operations, and multiplication without error. For Self-attention lemmas, we directly adapt from \cite{cottheory}.
\paragraph{Lemmas for MLP with ReGLU activation}
This section shows several lemmas showing the capabilities of the self-attention operation and MLP layers to approximate high-level vector operations. These high-level operations are later used as building blocks for the Transformer SAT-solver. Specifically, with appropriate weight configurations, a 2-layer MLP with ReGLU activation $f(\vx)=\mW_2[(\mW_1\vx  + \vb) \otimes \operatorname{relu}(\mV \vx + \vc)]$ can approximate the following vector operations for arbitrary input $\vx$:
\begin{itemize}
    \item Simulate a 2-layer MLP with ReLU activation: $\mW_2\operatorname{ReLU}(\mW'_1\vx+\vb'_1)+\vb'_2$
    \item Simulate any linear operation $\mW\vx$
    \item Simulate element-wise multiplication: $\vx_1\otimes \vx_2$
\end{itemize}

\begin{lemma}[Simulating a 2-Layer ReLU MLP with ReGLU Activation]
\label{lemma:relu_to_reglu}
A 2-layer MLP with ReGLU activation function can simulate any 2-layer MLP with ReLU activation function.
\end{lemma}

\begin{proof}
Let the ReLU MLP be defined as:
\[
g(\vx) = \mW'_2\, \operatorname{ReLU}(\mW'_1 \vx + \vb'_1) + \vb'_2.
\]

Set the weights and biases of the ReGLU MLP as follows:
\begin{align*}
\mW_1 &= \mathbf{0}, \quad \vb_1 = \mathbf{1}, \\
\mV &= \mW'_1, \quad \vb_2 = \vb'_1, \\
\mW_2 &= \mW'_2, \quad \vb = \vb'_2.
\end{align*}

Then, the ReGLU MLP computes:
\[
f(\vx) = \mW'_2 \left[ (\mathbf{0} \cdot \vx + \mathbf{1}) \otimes \operatorname{ReLU}(\mW'_1 \vx + \vb'_1) \right] + \vb'_2.
\]

Simplifying:
\[
f(\vx) = \mW'_2 \left[ \mathbf{1} \otimes \operatorname{ReLU}(\mW'_1 \vx + \vb'_1) \right] + \vb'_2 = \mW'_2\, \operatorname{ReLU}(\mW'_1 \vx + \vb'_1) + \vb'_2 = g(\vx).
\]

Thus, the ReGLU MLP computes the same function as the ReLU MLP.
\end{proof}

\begin{lemma}[Simulating Linear Operations with ReGLU MLP]
\label{lemma:linear_with_reglu}
A 2-layer MLP with ReGLU activation can simulate any linear operation \( f(\vx) = \mW \vx + \vb \).
\end{lemma}

\begin{proof}
To compute a linear function using the ReGLU MLP, we can set the activation to act as a scalar multiplier of one. Set the weights and biases as:
\begin{align*}
\mW_1 &= \mW, \quad \vb_1 = \vb, \\
\mV &= \mathbf{0}, \quad \vb_2 = \mathbf{1}, \\
\mW_2 &= \mI, \quad \vb = \mathbf{0}.
\end{align*}
Here, \( \mI \) is the identity matrix.

Since \( \mV \vx + \vb_2 = \vb_2 = \mathbf{1} \), we have:
\[
\operatorname{ReLU}(\mV \vx + \vb_2) = \operatorname{ReLU}(\mathbf{1}) = \mathbf{1}.
\]
Then, the ReGLU MLP computes:
\[
f(\vx) = \mI \left[ (\mW \vx + \vb) \otimes \mathbf{1} \right] = \mW \vx + \vb.
\]
Thus, any linear operation can be represented by appropriately setting \( \mW_1 \), \( \vb_1 \), and \( \mW_2 \).
\end{proof}

\begin{lemma}[Element-wise Multiplication via ReGLU MLP]
\label{lemma:mlp_mult}
A 2-layer MLP with ReGLU activation can compute the element-wise multiplication of two input vectors \( \vx_1 \) and \( \vx_2 \), that is,
\[
f(\vx) = \vx_1 \otimes \vx_2,
\]
where \( \vx = [\vx_1; \vx_2] \) denotes the concatenation of \( \vx_1 \) and \( \vx_2 \).
\end{lemma}

\begin{proof}
Let $\vx = [\vx_1; \vx_2] \in \mathbb{R}^{2n}$, where $\vx_1, \vx_2 \in \mathbb{R}^n$.

Set the weights and biases:

\[
\begin{aligned}
\mW_1 &= \begin{bmatrix} \mI_n \\ \mI_n \end{bmatrix}, & \vb_1 &= \mathbf{0}_{2n}, \\
\mV &= \begin{bmatrix} \mI_n \\ -\mI_n \end{bmatrix}, & \vb_2 &= \mathbf{0}_{2n}, \\
\mW_2 &= \begin{bmatrix} \mI_n & -\mI_n \end{bmatrix}, & \vb &= \mathbf{0}_n.
\end{aligned}
\]

Compute:

\[
\begin{aligned}
\mW_1 \vx + \vb_1 &= \begin{bmatrix} \vx_1 \\ \vx_1 \end{bmatrix}, \\
\mV \vx + \vb_2 &= \begin{bmatrix} \vx_2 \\ -\vx_2 \end{bmatrix}, \\
\operatorname{ReLU}(\mV \vx + \vb_2) &= \begin{bmatrix} \operatorname{ReLU}(\vx_2) \\ \operatorname{ReLU}(-\vx_2) \end{bmatrix}.
\end{aligned}
\]

The element-wise product:

\[
(\mW_1 \vx + \vb_1) \otimes \operatorname{ReLU}(\mV \vx + \vb_2) = \begin{bmatrix} \vx_1 \otimes \operatorname{ReLU}(\vx_2) \\ \vx_1 \otimes \operatorname{ReLU}(-\vx_2) \end{bmatrix}.
\]

Compute the output:

\[
\begin{aligned}
f(\vx) &= \mW_2 \left[ (\mW_1 \vx + \vb_1) \otimes \operatorname{ReLU}(\mV \vx + \vb_2) \right] + \vb \\
&= \vx_1 \otimes \operatorname{ReLU}(\vx_2) - \vx_1 \otimes \operatorname{ReLU}(-\vx_2) \\
&= \vx_1 \otimes \left( \operatorname{ReLU}(\vx_2) - \operatorname{ReLU}(-\vx_2) \right) \\
&= \vx_1 \otimes \vx_2.
\end{aligned}
\]

Thus, the ReGLU MLP computes $f(\vx) = \vx_1 \otimes \vx_2$ without restrictions on $\vx_2$.
\end{proof}

\paragraph{Capabilities of the Self-Attention Layer}
\label{sec:sacap}
In this subsection, we provide 2 core lemmas on the capabilities of the self-attention layer from \cite{cottheory}. 

Let $n\in\mathbb N$ be an integer and let $\vx_1, \vx_2, \cdots, \vx_n$ be a sequence of vectors where $\vx_i=(\tilde \vx_i,r_i,1) \in [-M,M]^{d+2}$, $\tilde \vx_i\in\mathbb R^d$, $r_i\in\mathbb R$, and $M$ is a large constant. Let $\mK,\mQ,\mV\in\mathbb R^{d'\times (d+2)}$ be any matrices with $\|\mV\|_\infty\le 1$, and let $0<\rho,\delta <M$ be any real numbers. Denote $\vq_i=\mQ\vx_i$, $\vk_j=\mK\vx_j$, $\vv_j=\mV\vx_j$, and define the \emph{matching set} $\gS_i=\{j\leq i: |\vq_i\cdot \vk_j|\le \rho\}$. Equipped with these notations, we define two basic operations as follows:
\begin{itemize}[topsep=0pt,leftmargin=30pt]
\setlength{\itemsep}{0pt}
    \item COPY: The output is a sequence of vectors $\vu_1,\cdots,\vu_n$ with $\vu_i=\vv_{\mathrm{pos}(i)}$, where $\mathrm{pos}(i)=\operatorname{argmax}_{j\in \gS_i}{r_j}$.
    \item MEAN: The output is a sequence of vectors $\vu_1,\cdots,\vu_n$ with $\vu_i=\operatorname{mean}_{j\in \gS_i}{\vv_j}$.
\end{itemize}
\begin{assumption}
\label{ass:attention}[Assumption C.6 from \cite{cottheory}]
    The matrices $\mQ,\mK,\mV$ and scalars $\rho, \delta$ satisfy that for all considered sequences $\vx_1, \vx_2, \cdots, \vx_n$, the following hold:
    \begin{itemize}[topsep=0pt,leftmargin=30pt]
    \setlength{\itemsep}{0pt}
        \item For any $i,j\in [n]$, either $|\vq_i\cdot \vk_j|\le \rho$ or $\vq_i\cdot \vk_j\le -\delta$.
        \item For any $i,j\in[n]$, either $i=j$ or $|r_i- r_j|\geq \delta$.
    \end{itemize}
\end{assumption}

\cref{ass:attention} says that there are sufficient gaps between the attended position (e.g., $\mathrm{pos}(i)$) and other positions. The two lemmas below show that the attention layer with casual mask can implement both COPY operation and MEAN operation efficiently.

\begin{lemma}[Lemma C.7 from \cite{cottheory}]
\label{lemma:TM_copy}
     Assume \cref{ass:attention} holds with $\rho\le\frac{\delta^2}{8M}$. For any $\epsilon > 0$, there exists an attention layer with embedding size $O(d)$ and one causal attention head that can approximate the COPY operation defined above. Formally, for any considered sequence of vectors $\vx_1, \vx_2, \dots, \vx_n$, denote the corresponding attention output as $\vo_1, \vo_2, \dots, \vo_n$. Then, we have $\|\vo_i-\vu_i\|_{\infty}\le\epsilon$ for all $i\in [n]$ with $\gS_i\neq \emptyset$. Moreover, the $\ell_\infty$ norm of attention parameters is bounded by $O(\mathrm{poly}(M,1/\delta,\log(n),\log(1/\epsilon)))$.
\end{lemma}

\begin{lemma}[Lemma C.8 from \cite{cottheory}]
\label{lemma:TM_mean}
    Assume \cref{ass:attention} holds with 
    $\rho\le\frac{\delta\epsilon}{16M\ln(\frac{4Mn}{\epsilon})}$.
    For any $0<\epsilon \le M$, there exists an attention layer with embedding size $O(d)$ and one causal attention head that can approximate the MEAN operation defined above. Formally, for any considered sequence of vectors $\vx_1, \vx_2, \dots, \vx_n$, denote the attention output as $\vo_1, \vo_2, \dots, \vo_n$. Then, we have $\|\vo_i-\vu_i\|_{\infty}\le\epsilon$ for all $i\in [n]$ with $\gS_i\neq \emptyset$. Moreover, the $\ell_\infty$ norm of attention parameters is bounded by $O(\mathrm{poly}(M,1/\delta,\log(n),\log(1/\epsilon)))$.
\end{lemma}

\subsection{Saturated Attention}
To introduce our construction of Transformer layers and attention head, we first introduce \textit{saturated self-attention}, which is an idealization of the usual softmax attention head that allows for sparse and uniform attention on previous positions:
\begin{definition}[Saturated Masked Attention,~\citet{saturated}]
\label{def:saturated98}
A saturated attention head with hidden dimension $d_h$, embedding dimension $d_{emb}$ and weight $\Gamma_s=(\mW_Q, \mW_K, \mW_V)$ is a function $\operatorname{SaturatedAttn}(\mX; \Gamma_{s}):\sR^{n\times \demb}\rightarrow \sR^{n\times \dh}$ that satisfy the following:
\begin{align*}
\mA:=\mX \mW_Q  (\mW_K\mX)^\top \in \sR^{n\times n}\\
\gM_i:=\{j\in[i]|\mA_{ij}=\max_k \mA_{ik}\}\\
\operatorname{SaturatedAttn}(\mX; \Gamma_{s})_i:=\frac{\sum_{j\in \gM_i} \mX_j \mW_V}{|\gM_i|}
\end{align*}
\end{definition}
Intuitively, while softmax attention computes a distribution of attention over all previous positions and computes a weighted average, saturated attention only attends to the previous positions with the highest attention value and computes a uniform average over these positions.

We now show that Saturated Attention can be approximated by normal softmax attention:

\begin{corollary}[Softmax Attention Can Approximate Saturated Attention, implied by \cref{lemma:TM_mean}]
\label{lemma:softmax_approx_saturated}
Let $n \in \mathbb{N}$.  
Consider any input sequence $\mX \in \mathbb{R}^{n \times d_{\mathrm{emb}}}$, and let 
$\operatorname{SaturatedAttn}(\mX;\Gamma_s)$ be a saturated attention head with a causal mask and parameter norm bounded by $O(1)$
that produces outputs $\mathbf{o}_1,\dots,\mathbf{o}_n \in \mathbb{R}^{d_h}$.  

Suppose further that, for each row $i$, 
the maximum attention score $\max_{j \le i} \bigl(\mA_{ij}\bigr)$ 
of the saturated head 
exceeds all other scores by a margin of at least $\delta>0$, i.e.\ 
if $j \in \gM_i$ (the set of maximizing indices) and $k \notin \gM_i$, then
$\mA_{ij} - \mA_{ik} \ge \delta.$

Then for any $\varepsilon>0$, 
there exists a \emph{standard single-head softmax attention} function 
$\mathrm{Attn}(\mX;\Gamma)$ 
with parameter norms bounded by 
$\mathrm{poly}\bigl(M,1/\delta,\log(n),\log(1/\varepsilon)\bigr)$
such that its outputs 
$\tilde{\mathbf{o}}_1,\dots,\tilde{\mathbf{o}}_n \in \mathbb{R}^{d_h}$ 
satisfy
\[
\bigl\|\tilde{\mathbf{o}}_i - \mathbf{o}_i\bigr\|_{\infty}
\;\le\;
\varepsilon
\quad
\text{for all }\,1\le i \le n.
\]

\noindent

\end{corollary}

In other words, 
\emph{if a saturated attention head has a strict dot-product margin among the top positions, 
it can be approximated arbitrarily closely by an ordinary causal softmax attention mechanism, 
using parameter magnitudes that grow at most polynomially in $1/\delta$, $M$, $\log(n)$, 
and $\log(1/\varepsilon)$.}

\begin{proof}

We rely on scaling arguments from standard “hard-max vs.\ softmax” approximations:
\begin{enumerate}
\item \textbf{Queries/Keys.}
Use scaled copies of $\mW_Q,\mW_K$ so that
\[
\mathbf{q}_i' \;=\;\alpha \,\bigl(\mW_Q\,\vx_i\bigr),
\quad
\mathbf{k}_j' \;=\;\alpha \,\bigl(\mW_K\,\vx_j\bigr),
\]
for some large $\alpha\gg \frac{1}{\delta}$ to amplify dot-product differences.  
By multiplying the entire query/key spaces by $\alpha$, 
the difference $\mA_{ij}-\mA_{ik}\ge\delta$ becomes
\[
\alpha\bigl(\mA_{ij}-\mA_{ik}\bigr)
\;\ge\;
\alpha\,\delta.
\]
Choosing $\alpha = O\!\bigl(\tfrac{1}{\delta}\log(\tfrac{n}{\varepsilon})\bigr)$ ensures that
\(\exp\bigl(\alpha\,\mA_{ij}\bigr)\) 
is exponentially larger than 
\(\exp\bigl(\alpha\,\mA_{ik}\bigr)\)
whenever $j\in \gM_i$ and $k\notin \gM_i$.  
Hence, positions in $\gM_i$ dominate the softmax distribution.

\item \textbf{Values.}
Set $\mW_V' = \mW_V$ (possibly scaled if needed so that $\|\mW_V'\|_\infty$ remains bounded by $\mathrm{poly}(M,1/\delta)$). 
Then at row $i$, the sum of vectors from $j\in\gM_i$ 
will approximate a uniform average, provided the softmax normalizes those positions evenly.  
If needed, small perturbations to $\mW_V'$ can ensure that each dimension remains $\leq 1$ in $\ell_\infty$ norm.
\end{enumerate}

Under this construction, for each row $i$, the softmax $\alpha_{ij}$ assigns $j\in\gM_i$ almost the same weight (because $\mA_{ij}$ differ by less than $\delta$ among $j\in\gM_i$) and assigns $k\notin\gM_i$ exponentially smaller weights.  The ratio between the largest and second‐largest exponent is at least $\exp(\alpha\,\delta)$.  By choosing $\alpha$ such that $\exp(\alpha\,\delta)\ge \tfrac{n}{\varepsilon}$, positions $k\notin\gM_i$ contribute at most $\varepsilon/n$ fraction of the total probability.

Consequently, the softmax distribution is $\varepsilon$-close to “uniform over $\gM_i$” in total variation.  Multiplying by $\mX_j \mW_V'$ then yields 
\[
\|\tilde{\mathbf{o}}_i - \mathbf{o}_i\|_\infty 
\;\le\; 
\varepsilon
\]
by a standard convex combination argument (the difference in expected values under two distributions that differ by $\varepsilon$ in total variation is at most $\varepsilon$ times the largest possible difference in outcomes, and $\|\mW_V'\|_\infty = O(\mathrm{poly}(M,\frac1\delta))$).

Finally, we note that each weight (coordinate in $\mW_Q',\mW_K',\mW_V'$) is at most $O(\alpha\,M)$ in absolute value, plus any minor adjustments.  Since $\alpha=O\!\bigl(\frac{1}{\delta}\log(\frac{n}{\varepsilon})\bigr)$ and $M$ is the original data bound, the overall parameter norms are bounded by $\mathrm{poly}\bigl(M,\frac{1}{\delta},\log(n),\log(\frac1\varepsilon)\bigr)$.

Putting all these steps together proves that we can approximate each saturated attention output $\mathbf{o}_i$ by a standard causal softmax attention output $\tilde{\mathbf{o}}_i$ to within $\|\cdot\|_{\infty}$ error $\le\varepsilon$.  
\end{proof}

\subsection{Proof of \cref{lemma:parallel}}




We proof a version of \cref{lemma:parallel} that uses saturated attention. \cref{lemma:parallel} is immediately implied by the following lemma and \cref{lemma:softmax_approx_saturated}
\begin{lemma}[Saturated Masked Attention version of \cref{lemma:parallel}]
\label{lemma:saturated_parallel}
    Let $F$ be a 3-SAT formula over variables $\{x_1, \dots, x_p\}$ with $c$ clauses $\{C_1, \dots, C_c\}$ and $A$ a partial assignment defined on variables $\{x_1, \dots, x_p\}$. Let \[
\mX_{encoding} =
\begin{bmatrix}
0 & 1 & 1 \\
E(C_1) & 0 & 1 \\
\vdots & \vdots & \vdots \\
E(C_c) & 0 & 1 \\
E(A) & 0 & 1
\end{bmatrix} \in \mathbb{R}^{(c+2)\times (2p+2)}
\] 
Then given $X$ as input, there exists:
\begin{itemize}
    \item An saturated attention head with parameters $\Gamma_s^{A\models F}$ and hidden dimension $1$ that satisfies
    $$\operatorname{SaturatedAttn}(\mX; \Gamma_s^{A\models F})_{c+2}=\ind_{A\models F}$$
    \item An saturated attention head with parameters $\Gamma_s^{F\models \lnot A}$ and hidden dimension $1$ that satisfies
    $$\operatorname{SaturatedAttn}(\mX; \Gamma_s^{F\models \lnot A})_{c+2}=\ind_{F\models \lnot A}$$
    \item An saturated attention head with parameters $\Gamma_s^{D}$ with hidden dimension $2p$ and MLP layer with parameters $\Gamma_{MLP}^{D}$ satisfy:
    $$MLP([\operatorname{SaturatedAttn}(\mX; \Gamma_s^{D}); \mX]; \Gamma_{MLP}^{D})_{c+2}=E(D)$$
    unless $F\models \lnot A$, where $E(D)$ is as defined in \ref{lemma:vec_ded}
\end{itemize}
\end{lemma}
\begin{proof}
We prove each of the three constructions in turn, using the definition of saturated attention (\cref{def:saturated98}) and standard reductions from the logical semantics to dot-product comparisons.

We explain how to construct parameter matrices $(\mW_Q,\mW_K,\mW_V)$ such that the resulting saturated attention head implements:
\begin{enumerate}
\item a check for $\ind_{A \models F}$ (i.e.\ whether $A$ satisfies $F$),
\item a check for $\ind_{F \models \lnot A}$ (i.e.\ whether $A$ contradicts $F$),
\item a step of unit propagation that yields $E(D)$, provided $F \not\models \lnot A$.
\end{enumerate}

Within the following proof of \cref{lemma:saturated_parallel}, we shorten $\mX_{encoding}$ as $\mX$.

\subsection*{1.\ Checking Satisfiability ($A \models F$)}
\noindent
We construct the matrices 
\[
\mW_Q^{A\models F} \in \sR^{(2p+2)\times(2p+1)}, 
\quad 
\mW_K^{A\models F} \in \sR^{(2p+2)\times(2p+1)}, 
\quad 
\mW_V^{A\models F} \in \sR^{(2p+2)\times 1}
\]
as follows (with block‐wise or coordinate‐wise \(\mathbf{0}\) and \(\mathbf{0}_{2p}\) denoting matrices/vectors of all zeros of dimension $2p$ where the dimension subscript is omitted if they can be inferred from other entries, and \(\mI_{2p}\) the \(2p\times 2p\) identity matrix). 
\[
\mW_Q^{A\models F} 
\;=\;
\begin{bmatrix}
\mI_{2p} & 0
\\[6pt]
\mathbf{0}^\top & 0
\\[6pt]
\mathbf{0}^\top & 1
\end{bmatrix}
\qquad
\mW_K^{A\models F} 
\;=\;
\begin{bmatrix}
-\mI_{2p} & 0
\\[6pt]
\mathbf{0}^\top & -0.5
\\[6pt]
\mathbf{0}^\top & 0
\end{bmatrix}
\qquad
\mW_V^{A\models F} 
\;=\;
\begin{bmatrix}

\mathbf{0}_{2p}
\\[6pt]
1
\\[6pt]
0
\end{bmatrix}.
\]

Then

\[
\mX \mW_Q^{A\models F} =
\begin{bmatrix}
\mathbf{0}_{2p} & 1
\\[6pt]
E(C_1) & 1
\\
\vdots & \vdots
\\
E(C_c) & 1
\\[6pt]
E(A) & 1
\end{bmatrix}
\quad
\mX \mW_K^{A\models F} =
\begin{bmatrix}
\mathbf{0}_{2p} & -0.5
\\[6pt]
- E(C_1) & 0
\\
\vdots & \vdots
\\
- E(C_c) & 0
\\[6pt]
- E(A) & 0
\end{bmatrix}\quad
\mX \mW_V^{A\models F} =
\begin{bmatrix}
1
\\[6pt]
0
\\
\vdots
\\
0
\end{bmatrix}
\]


\[
\mA:=\mX \mW_Q^{A\models F}  (\mW_K^{A\models F}\mX)^\top = 
\begin{bmatrix}
-0.5 & 0 & 0 & \dots & 0
\\[6pt]
-0.5 & -E(C_1) \cdot E(C_1) & -E(C_1) \cdot E(C_2) &\dots & -E(C_1) \cdot E(A)
\\
-0.5 & -E(C_2) \cdot E(C_1) & -E(C_2) \cdot E(C_2) &\dots & -E(C_2) \cdot E(A)
\\
\vdots & \vdots & \vdots & & \vdots
\\
-0.5 & -E(A) \cdot E(C_1) & -E(A) \cdot E(C_2) & \dots & -E(A) \cdot E(A)
\end{bmatrix}
\]

Since we want to output $\ind_{A\models F}$ at the last position $c+2$ corresponding to $E(A)$ in $\mX_{encoding}$, we focus on the last row of $A$:
\[
    \mA_{c+2}=[-0.5 \quad -E(A) \cdot E(C_1) \quad-E(A) \cdot E(C_2) \quad \dots \quad -E(A) \cdot E(C_c)\quad -E(A) \cdot E(A)]
\]

Now consider $\gM_{c+2}=\{j\in[c+2]|\mA_{(c+2),j}=\max_k \mA_{(c+2),k}\}$. Note that $\forall i\in[c],E(A) \cdot E(C_i)\in \sN$ and since $A_{(c+2),1}=-0.5$ there is:
$$ \gM_{c+2}=\{1\} \quad \Longleftrightarrow \quad \min_{i\in [c]} E(C_i) \cdot E(A) \geq 1.$$
$$ \gM_{c+2}\subset[2,c+2] \quad \Longleftrightarrow \quad \min_{i\in [c]} E(C_i) \cdot E(A) = 0.$$
which are the only 2 possibilities for nonnegative integers $E(C_i) \cdot E(A)$. Also, since $(\mX \mW_V^{A\models F})^\top=[1\ 0\ 0\ \dots \ 0]$ we have that
$$\mX_j\mW_V^{A\models F}=\begin{cases}
    1 & \text{if } j=1\\
    0 & \text{otherwise}
\end{cases}$$

\begin{align*}
    \operatorname{SaturatedAttn}(\mX; \Gamma_{s})_{c+2}&:=\frac{\sum_{j\in \gM_{c+2}} \mX_j \mW_V^{A\models F}}{|\gM_{c+2}|}\\&=\begin{cases}
    \frac{1}{1} & \text{if }\gM_{c+2}=\{1\}\\
    0 & \text{if }\gM_{c+2}\subset[2,c+2]
\end{cases}\\&=\ind_{\gM_{c+2}=\{1\}}\\&=\ind_{\min_{i\in [c]} E(C_i) \cdot E(A) \geq 1}\\
&=\ind_{A\models F}
\end{align*}
where the last step is by \cref{lemma:vec_ded}. This concludes our proof for satisfiability checking.

\subsection*{2.\ Detecting Conflict ($F \models \lnot A$)}
Note that for $B\in\gB$ we have
$$
\enf(B)= \begin{bmatrix}
    \mathbf{0}_{p\times p} & -\mI_p\\
    -\mI_p & \mathbf{0}_{p\times p}\\
    \end{bmatrix} E(B)+\mathbf{1}_p
$$
Define
$$\mP_{\text{not-false}}:=\begin{bmatrix}
    \mathbf{0}_{p\times p} & -\mI_p & \mathbf{0}_p & \mathbf{0}_p\\
    -\mI_p & \mathbf{0}_{p\times p} & \mathbf{0}_p & \mathbf{0}_p\\
    \mathbf{0}_{p}^\top & \mathbf{0}_{p}^\top & 1 & 0\\
    \mathbf{1}_{p}^\top & \mathbf{1}_{p}^\top & 0 & 1
\end{bmatrix}\in \sR^{(2p+2)\times (2p+2)}$$
Then
$$\mX\mP_{\text{not-false}}=
\begin{bmatrix}
0 & 1 & 1 \\
\enf(C_1) & 0 & 1 \\
\vdots & \vdots & \vdots \\
\enf(C_c) & 0 & 1 \\
\enf(A) & 0 & 1
\end{bmatrix} \in \mathbb{R}^{(c+2)\times (2p+2)}
$$

We now construct the matrices 
\[
\mW_Q^{F\models\lnot A} \in \sR^{(2p+2)\times(2p+1)}, 
\quad 
\mW_K^{F\models\lnot A} \in \sR^{(2p+2)\times(2p+1)}, 
\quad 
\mW_V^{F\models\lnot A} \in \sR^{(2p+2)\times 1}
\]
as follows:
\[
\mW_Q^{F\models\lnot A} 
\;=\;\mP_{\text{not-false}}\mW_Q^{A\models F} 
\qquad
\mW_K^{F\models\lnot A}
\;=\;
\mW_K^{A\models F}
\qquad
\mW_V^{F\models\lnot A} 
\;=\;
\begin{bmatrix}

\mathbf{0}_{2p}
\\[6pt]
-1
\\[6pt]
1
\end{bmatrix}.
\]

Then

\[
\mX \mW_Q^{F\models\lnot A} = 
\begin{bmatrix}
\mathbf{0}_{2p} & 1
\\[6pt]
\enf(C_1) & 1
\\
\vdots & \vdots
\\
\enf(C_c) & 1
\\[6pt]
\enf(A) & 1
\end{bmatrix}
\quad
\mX \mW_K^{F\models\lnot A} =
\begin{bmatrix}
\mathbf{0}_{2p} & -0.5
\\[6pt]
- E(C_1) & 0\\
\vdots & \vdots
\\
- E(C_c) & 0
\\[6pt]
- E(A) & 0
\end{bmatrix}\quad
\mX \mW_V^{F\models\lnot A} =
\begin{bmatrix}
0
\\[6pt]
1
\\
\vdots
\\
1
\end{bmatrix}
\]

Recall from \cref{lemma:vec_ded} that:
\[
F \models \lnot A \quad \Longleftrightarrow \quad \min_{i \in [c]} \Big(E(C_i) \cdot \enf(A)\Big) = 0.
\]

The remaining argument is very similar to satisfiability checking and we omit the full proof.

\subsection*{3.\ Unit Propagation ($D$)}

Recall that $D:=\{l \in L\;|\;F\land A\models_1 l\}$ and 
\begin{equation}
    \label{eq:d_orig}
    E(D)=\max\Bigl[ \min\Bigl(\sum_{i \in [c]} 
            E(C_i) \boldsymbol {1}_{\{E(C_i)\cdot \enf(A)=1\}}, 1\Bigr)
            -  \eass(A),\; 0 \Bigr].
\end{equation}

To address unit propagation with saturated attention, we use a slightly different formulation than the formula in \cref{lemma:vec_ded}: 
\begin{proposition}
Let $m>1$ be an arbitrary constant, then \begin{align*}
    \vz&:=\sum_{i\in[c]}\boldsymbol{1}_{\{E(C_i)\cdot \enf(A)=1\}}\cdot E(C_i)\\
    E(D)&=\operatorname{ReLU}(m\vz-\eass(A))-\operatorname{ReLU}(m\vz-1)
\end{align*}
\end{proposition}

\begin{proof}
We start from the expression in \eqref{eq:d_orig},
\[
E(D) 
\;=\;
\max\!\Bigl[\,
\min\!\bigl(\vz,\,1\bigr)\;-\;\eass(A),\;0
\Bigr],
\quad
\text{where}
\quad
\vz \;:=\; 
\sum_{i\in[c]}\boldsymbol{1}_{\{\,E(C_i)\cdot \enf(A)=1\}}\;\cdot\;E(C_i).
\]
Because $\eass(A)\in\{0,1\}^{2p}$, each coordinate of $\eass(A)$ is either $0$ or $1$.  A straightforward elementwise check shows the identity
\[
\max\!\bigl(\,\min(a,1)\;-\;b,\,0\bigr)
\;=\;
\mathrm{ReLU}\!\bigl(ma - b\bigr) \;-\;\mathrm{ReLU}(ma-1),
\]
whenever $b\in\{0,1\}$.  Indeed:
\begin{itemize}
\item If $b=0$, then the left side is $\max(\min(a,1),0)$; on the right side,
\[
\mathrm{ReLU}(ca) - \mathrm{ReLU}(ca-1)
\]
exactly matches $\max(\min(ma,1),0)=\max(\min(a,1),0)$ for any $a\geq 1$ (this is a standard piecewise identity).
\item If $b=1$, then $\min(a,1)-1 \le 0$, hence the left side is always $0$.  
On the right side, 
\[
\operatorname{ReLU}(ma - 1) \;-\; \operatorname{ReLU}(ma-1) \;=\; 0.
\]
\end{itemize}
Applying this identity coordinatewise, we obtain
\[
\max\!\Bigl[\,
\min(c\vz,1)\;-\;\eass(A),\; 0
\Bigr]
\;=\;
\mathrm{ReLU}(m\vz - \eass(A))
\;-\;
\mathrm{ReLU}(m\vz - 1),
\]
which matches the stated expression for $E(D)$.
\end{proof}

We now construct the matrices 
\[
\mW_Q^{D} \in \sR^{(2p+2)\times(2p+1)}, 
\quad 
\mW_K^{D} \in \sR^{(2p+2)\times(2p+1)}, 
\quad 
\mW_V^{D} \in \sR^{(2p+2)\times (2p)}
\]
as follows:
\[
\mW_Q^{D} 
\;=\;\mW_Q^{F\models\lnot A}
\qquad
\mW_K^{D}
\;=\;
\begin{bmatrix}
-\mI_{2p} & 0
\\[6pt]
\mathbf{0}^\top & -1.5
\\[6pt]
\mathbf{0}^\top & 0
\end{bmatrix}\quad
\mW_V^{D} 
\;=\;c
\begin{bmatrix}
\mI_p\\
\mathbf{0}_p^\top\\
\mathbf{0}_p^\top
\end{bmatrix}.
\]
Then

\[
\mX \mW_Q^{D} = 
\begin{bmatrix}
\mathbf{0}_{2p} & 1
\\[6pt]
\enf(C_1) & 1
\\
\vdots & \vdots
\\
\enf(C_c) & 1
\\[6pt]
\enf(A) & 1
\end{bmatrix}
\quad
\mX \mW_K^{D} =
\begin{bmatrix}
\mathbf{0}_{2p} & -1.5
\\[6pt]
- E(C_1) & 0\\
\vdots & \vdots
\\
- E(C_c) & 0
\\[6pt]
- E(A) & 0
\end{bmatrix}\quad
\mX \mW_V^{D} =c
\begin{bmatrix}
\mathbf{0}_p
\\[6pt]
E(C_1)
\\
\vdots
\\
E(A)
\end{bmatrix}
\]

We focus on the last row of $\mA:=\mX \mW_Q^D  (\mW_K^D\mX)^\top$:

$$\mA_{c+2}=[-1.5 \quad -E(A) \cdot \enf(C_1) \quad-E(A) \cdot \enf(C_2) \quad \dots \quad -E(A) \cdot \enf(C_c)\quad -E(A) \cdot \enf(A)]$$

Also, recall that we assume here $F\not\models \lnot A$, so $\forall i, E(A) \cdot \enf(C_i)\geq 1$ and therefore $E(A) \cdot \enf(C_i)$ are positive integers. :
$$ \gM_{c+2}=\{1\} \quad \Longleftrightarrow \quad \min_{i\in [c]} E(C_i) \cdot E(A) \geq 2.$$
$$ \gM_{c+2}\subset[2,c+2] \quad \Longleftrightarrow \quad \min_{i\in [c]} E(C_i) \cdot E(A) = 1.$$

In particular:

$$\gM_{c+2}=\begin{cases}
    \{1\} & \text{if } \min_{i\in [c]} E(C_i) \cdot \eass(A) \geq 2\\
    \{j\in[c]|E(C_i) \cdot \eass(A)=1\} & \text{otherwise}
\end{cases}$$

As a result:
\begin{align*}
    \operatorname{SaturatedAttn}(\mX; \Gamma_{s})_{c+2}&:=\frac{\sum_{j\in \gM_{c+2}} \mX_j \mW_V^{D}}{|\gM_{c+2}|}\\
    &=\begin{cases}
        \mathbf{0}_{2p} & \text{if }\gM_{c+2}=\{1\}\\
        \frac{c}{|\gM_{c+2}|}\sum_{i\in[c]}\boldsymbol{1}_{\{E(C_i)\cdot \enf(A)=1\}}\cdot E(C_i) & \text{if }\gM_{c+2}\subset[2,c+2]
    \end{cases}\\
    &=m\sum_{i\in[c]}\boldsymbol{1}_{\{E(C_i)\cdot \enf(A)=1\}}\cdot E(C_i)\\
&=m\vz
\end{align*}
for $m=\frac{c}{|\gM_{c+2}|}>1$.

We now construct the weights for the ReGLU MLP layer. By \cref{lemma:relu_to_reglu} we know that ReGLU MLP can simulate ReLU MLPs. Therefore, we only need to construct $\mW_1^D,\mW_2^D,\vb_1^D,\vb_2^D$ such that \[
\mW^D_2\, \operatorname{ReLU}(\mW^D_1 [m\vz;\mX_{c+2}] + \vb^D_1) + \vb^D_2=\mathrm{ReLU}(m\vz - \eass(A))
\;-\;
\mathrm{ReLU}(m\vz - 1).
\]
Note that $\mX_{c+2}=[E(A)\quad 0\quad 1]$, therefore $[m\vz;\mX_{c+2}]\in\sR^{4p+2}$
Also,
$$\eass(A)=\begin{bmatrix}
    \mI_p & \mI_p \\
    \mI_p & \mI_p 
\end{bmatrix}E(A)$$

Therefore, define
\[
\mW_1^D =
\begin{bmatrix}
\mI_{2p} & \mathbf{0}_{2p \times 2p} & -\begin{bmatrix} \mI_p & \mI_p \\ \mI_p & \mI_p \end{bmatrix} & \mathbf{0}_{2p \times 2} \\[6pt]
\mI_{2p} & \mathbf{0}_{2p \times 2p} & \mathbf{0}_{2p \times 2p} & \mathbf{0}_{2p \times 2}
\end{bmatrix}
\]

\[
\vb_1^D =
\begin{bmatrix}
\mathbf{0}_{2p} \\[6pt] -\mathbf{1}_{2p}
\end{bmatrix}
\]

\[
\mW_2^D =
\begin{bmatrix}
\mI_{2p} & -\mI_{2p}
\end{bmatrix}
\]

\[
\vb_2^D = \mathbf{0}_{2p}
\]

It can be easily verified that this satisfies the desired equality.
\end{proof}

\subsection{Theoretical Construction (\cref{thm:sat_search})}



\label{sec:construction}
\paragraph{Notations}
\begin{itemize}
    \item $p$ denotes the number of variables
    \item $t_i$ denotes the token at position $i$
    \item $T_{vars}$ denotes the set of tokens that denote variables and their negations. i.e. `\texttt{1}', `\texttt{2}', $\dots$, `\texttt{n}', `\texttt{-1}', `\texttt{-2}', $\dots$, `\texttt{-n}'
    \item $b$ denotes boolean variables
\end{itemize}
\begin{proof} 

    We first describe the encoding format of the formulas and the solution trace format before going into the details of model construction.

    \paragraph{Input Format.} We consider 3-CNF-SAT formulas in the DIMACS representation, with an initial \texttt{[BOS]} token and an ending \texttt{[SEP]} token. Each variable \( x_i \) for \( i \in [n] \) has 2 associated tokens: \texttt{i} and \texttt{-i} (e.g., \texttt{1} and \texttt{-1}), where the positive token indicates that the \( i \)-th variable appears in the clause while the negative token indicates that the negation of the \( i \)-th variable appears in the clause. Clauses are separated using the \texttt{0} token. For example, the formula

\begin{align*}
(\lnot x_2 \lor \lnot x_4 \lor \lnot x_1) \land (x_3 \lor x_4 \lor \lnot x_1) \land (\lnot x_1 \lor \lnot x_3 \lor \lnot x_2)\\ \land (x_1 \lor \lnot x_2 \lor \lnot x_4) \land (\lnot x_4 \lor x_2 \lor x_1) \land (x_1 \lor \lnot x_2 \lor x_4)
\end{align*}

would be represented as:
\begin{center}
    
\tok{[BOS] -2 -4 -1 0 3 4 -1 0 -1 -3 -2 0 1 -2 -4 0 -4 2 1 0 1 -2 4 0 [SEP]}

\end{center}

    \paragraph{Solution Trace Format.} The trace keeps track of the order of the assignments made and whether each assignment is a decision (assumption) or a unit propagation (deduction). Literals with a preceding \tok{D} token are decision literals while other literals are from unit propagation. When the model encounters a conflict between the current assignment and the formula, it performs a backtrack operation denoted by \tok{[BT]} and performs another attempt with the last decision literal negated. In particular, compared to \cref{fig:cot}, we used \tok{D} to abbreviate \tok{Assume} and use \tok{[BT]} to abbreviate \tok{Backtrack}\\
    \par As an example, the solution trace for the above SAT formula would be:\\
    \texttt{[SEP] D 2 D 1 -4 3 [BT] D 2 D -1 -4 [BT] -2 D 3 D 4 -1 SAT}
    We use simplified versions of the tokens compared to \cref{fig:cot}. In particular, we use $\texttt{[BT]}$ as a shorthand for $\tok{BackTrack}$ and $\tok{D}$ for $\tok{Deduce}$.

    \paragraph{Embedding Layer.} Our token set consists of one token for each variable and its negation, the separator token \texttt{0}, and a special token \texttt{D} to denote where decisions are made. The positional encoding occupies a single dimension and contains the numerical value of the position of the token in the string. (i.e. there exists a dimension $pos$ such that the position embedding of position $i$ is $i\cdot \mathbf{e}_{pos}$)
    \paragraph{Layer 1.} The first layer prepares for finding the nearest separator token and $D$ token. Let $i$ denote the position index of tokens:
    \begin{enumerate}
        \item Compute $i_{\text{sep}}$ where $i_{\text{sep}}=i$ if the corresponding token $t_i\in \{\text{`\texttt{0}', `\texttt{[SEP]}', `\texttt{[BT]}'}\}$ and $i_{\text{sep}}=0$ otherwise
        \item Similarly, compute $i_{\tok{D}}$ where $i_{\tok{D}}=i$ if the corresponding token $t_i=\tok{D}$ and $i_{\text{sep}}=0$ otherwise.
        \item Compute $(i-1)^2$, $i^2$ for index equality comparison
    \end{enumerate}
    The first 2 operations can both be computed using a single MLP layer that multiplies between $i$ from the positional encoding using \cref{lemma:mlp_mult}. Similarly, the 3rd operation is a multiplication operation that can be performed with \cref{lemma:mlp_mult}.
    \paragraph{Layer 2.} This layer uses 2 heads to perform the following tasks:
    \begin{enumerate}
        \item Copy the index and type of the last separator token and stores 
        \begin{align*}
            p_i^{sep}{}'&=\max\{j:j\leq i,t_j\in \{\text{`\texttt{0}', `\texttt{[SEP]}', `\texttt{[BT]}'}\}\}\\
            b_\texttt{0}&=(t_j=\text{`\texttt{0}'})\\
            b_\texttt{[SEP]}&=(t_j=\text{`\texttt{[SEP]}'})\\
            b_\texttt{[BT]}&=(t_j=\text{`\texttt{[BT]}'})\\
        \end{align*}
        for $j=p_i^{sep}{}'$
        \item (Backtrack) Compute the position of the nearest \texttt{D} token $p_i^\texttt{D}=\max\{j:j\leq i,t_j=\text{`\texttt{D}'}\}$
        \item Compute $(p_i^{sep}{}')^2$ for index operation
    \end{enumerate}
    \par Task 1 can be achieved via the COPY operation from \cref{lemma:TM_copy} with $\vq_i=1$, $\vk_i=i_\text{sep}$, $\vv_j=(j, \mathbb{I}[t_j=\text{`\texttt{0}'}], \mathbb{I}[t_j=\text{`\texttt{[SEP]}'}], \mathbb{I}[t_j=\text{`\texttt{[UP]}'}], \mathbb{I}[t_j=\text{`\texttt{[BackTrack]}'}])$.
    \par Task 2 is highly similar to task 1 and can be achieved using COPY with $\vq_i=1$, $\vk_i=i_{\tok{D}}$, $\vv_j=(j)$
    \par Task 3 is a multiplication operation that can be performed using \cref{lemma:mlp_mult}.
    \paragraph{Layer 3} This layer uses 1 head to copy the several values from the previous token to the current token. Specifically, this layer computes:
    \begin{enumerate}
        \item The position of the {\it previous} separator token, not including the current position:$$p_i^{sep}=\max\{j:j< i,t_j\in \{\text{`\texttt{0}', `\texttt{[SEP]}',`\texttt{[UP]}', `\texttt{[BackTrack]}'}\}\}$$
        \item Dermine if the previous token is {\tt D}: $b_{decision}=(t_{i-1}=\text{`\texttt{D}'})$ i.e., whether the current token is a decision variable
        \item (Induction) Compute the offset of the current token to the previous separator token $d_i^{sep}=i-p_i^{sep}{}'$
        \item Compute $(p_i^{sep})^2$,  for equality comparison at the next layer.
    \end{enumerate}
    \par Task 1 and 2 is done by copying $p_i^{sep}{}'$ and $\mathbb I[t_i=\text{`\texttt{D}'}]$ from the previous token. Specifially, we use the COPY operation from \cref{lemma:TM_copy} with $\vq_i=((i-1)^2, i-1,1)$ and $\vk_j=(-1,2j,-j^2)$ which determines $i-1=j$ via $-((i-1)-j)^2=0$ and $\vv_j=(p_i^{sep}{}',\mathbb I[t_i=\text{`\texttt{D}'}])$.
    Task 4 is a local multiplication operation that can be implemented via \cref{lemma:mlp_mult}.
    \paragraph{Layer 4.} This layer uses 2 heads to perform the following tasks:
    \begin{enumerate}
        \item Compute the sum of all variable token embeddings after the previous separator to encode a vector representation of assignments and clauses at their following separator token. $$\mathbf{r}_i=\sum_{j > p_i^{sep}, t_j\in T_{vars}}\mathbf{e}_{id(t_j)}=\sum_{p_j^{sep}=p_i^{sep}, t_j\in T_{vars}}\mathbf{e}_{id(t_j)}$$
        \item (Induction) Compute the position of the second-to-last separator $p_i^{sep-}=\max\{j:j<p_i^{sep},t_j\in \{\text{`\texttt{0}', `\texttt{[SEP]}',`\texttt{[UP]}', `\texttt{[BackTrack]}}'\}\}=p_{p_i^{sep}{}'}^{sep}$ and the corresponding current position in the previous state $p_i^{-}=p_i^{sep-}+d_i^{sep}$. As a special case for the first state, we also add $4$ to $p_i^{-}$ if $b_\texttt{[SEP]}$ is true, i.e. $p_i^{-}=p_i^{sep-}+d_i^{sep}+4\cdot b_{\texttt{[SEP]}}$. The additional $4$ is the number of variables per clause + 1 to ensure that we don't consider the last clause as an assignment.
        \item (Backtrack) Compute the position of the nearest \texttt{D} token to the last separator token $p_i^{\texttt{D}-}=p_{p_{i}^{sep}{'}}^\texttt{D}$
        \item Compute $b_{exceed}=(p_i^-> p_i^{\texttt{D}-} + 1)$, this denotes whether we're beyond the last decision of the previous state.
        \item Compare $(p_i^\texttt{D-} \leq p_i^-)$ for $b_{\text{BT\_finished}}$ at the next layer.
        \item Compare if $p_i^\texttt{D-} = p_i^-$ for the $b_{backtrack}$ operator.
        \item Compute $b_{copy}'=(p_i^- < p_i^{sep}{}' - 1)$
    \end{enumerate}

    Task 1 is achieved using a MEAN operation with $\vq_i=((p_i^{sep})^2, p_i^{sep},1)$, $\vk_j=((-1, 2p_j^{sep},-(p_j^{sep})^2)$, $\vv_j=\mathbf{e}_{id(t_j)}$ for $t_j\in T_{vars}$. This attention operations results in $\frac{\mathbf{r}_i}{i-p_i^{sep}}$ The MLP layer then uses \cref{lemma:mlp_mult} to multiply the mean result by $i-p_i^{sep}$ to obtain the $\mathbf{r}_i$.

    Task 2 is achieved using the COPY operation with $\vq_i=((p_i^{sep})^2, p_i^{sep},1)$, $\vk_j=(-1,2j,-j^2)$ and $\vv_j=p_i^{sep}{}'$. The MLP layer then performs the addition operation the computes $p_i^-$ by \cref{lemma:linear_with_reglu}
    
    Similarly, Task 3 is achieved using the COPY operation with $\vq_i=((p_i^{sep})^2, p_i^{sep},1)$, $\vk_j=(-1,2j,-j^2)$ and $\vv_j=p_i^{\texttt{D}}$.

    \paragraph{Layer 5.} The third layer uses 5 heads to perform the following tasks:
    \begin{enumerate}
        \item Compute $\ind_{A\models F}$, $\ind_{F\models \lnot A}$, $E(D)$ where $D:=\{l \in L\;|\;F\land A\models_1 l\}$ according to \cref{lemma:vec_ded}
        \item Compute $b_{final}=b_{exceed}\land b_{decision}$
        \item Compare $b_{no\_decision}=(p_i^{\texttt{D}}\leq p_i^{sep})$, which denotes whether the current state contains {\it no} decision variables
        \item Compute $b_{\text{BT\_finished}}=(p_i^\texttt{D-} \leq p_i^-) \land b_{\texttt{[BackTrack]}}$
        \item Compare $p_i^{-}$ with $p_i^{\texttt{D}-}-1$ by storing $p_i^{-}\leq p_i^{\texttt{D}-}-1$ and $p_i^{-}\geq p_i^{\texttt{D}-}-1$ (to check for equality at the next layer)
        \item Compare $b_{\text{backtrack}}=(p_i^{-}=p_i^{\texttt{D}-}-1)$
        
    \end{enumerate}

    \paragraph{Layer 6}
    This layer does the remaining boolean operators required for the output. In particular,
    \begin{itemize}
        \item $b_{unsat} = b_{no\_decision} \land b_{cont}$
        \item $b_{\tok{[BT]}}=b_{cont}\land \lnot (t_i=\tok{[BT]})$
        \item Compute a vector that is equal to $b_{backtrack} \cdot \rve_{BT}$, which is equal to $\rve_{BT}$ if $b_{backtrack}$ is True and $\boldsymbol{0}$ otherwise. This is to allow the operation at the output layer for backtracking
    \end{itemize}
    Note that $\land$ can be implemented as a single ReLU operation for tasks 1 and 2 that can be implemented with \cref{lemma:relu_to_reglu}, and task 3 is a multiplication operation implemented with \cref{lemma:mlp_mult}

    \paragraph{Layer 7}
    This layer performs a single operation with the MLP layer: Compute $b_{copy}\cdot e_{copy}$, which gates whether $e_{copy}$ should be predicted based on $b_{copy}$. This enables condition 5 at the output layer.
    \paragraph{Output Projection}
    The final layer is responsible for producing the output of the model based on the computed output of the pervious layers. We constructed prioritized conditional outputs, where the model outputs the token according to the first satisfied conditional in the order below:
    \begin{enumerate}
        \item If $b_{sat}$ output \tok{SAT}
        \item If $b_{cont} \land b_{no\_decision}$ output \tok{UNSAT}
        \item If $b_{cont} \land \lnot (t_i=\tok{[BackTrack]})$ output `\texttt{[BackTrack]}'
        \item (BackTrack) If $b_{backtrack}$, output the negation of the token from position $p_i^{\texttt{D}-}+1$
        \item (Induction) If $b_{copy}$, copy token from position $p_i^{-}+1$ as output ($e_{copy}$)
        \item output a unit propagation variable, if any.
        \item output \texttt{D} if the current token is not \texttt{D}
        \item output a unassigned variable
    \end{enumerate}

    \par For the output layer, we use $l_{\texttt{[TOKEN]}}$ to denote the output logit of $\texttt{[TOKEN]}$. Since the final output of the model is the token with the highest logit, we can implement output priority by assigning outputs of higher priority rules with higher logits than lower priority rules. Specifically, we compute the output logits vector using the output layer linear transformation as:
    $$
2^7 \cdot b_{sat} \cdot \rve_{\texttt{SAT}} + 2^6 \cdot b_{cont} \cdot \rve_{\texttt{[BackTrack]}} + 2^5 \cdot b_{unsat} \cdot \rve_{\texttt{UNSAT}}$$ $$ + 2^4 \cdot b_{backtrack} \cdot \rve_{BT} + 2^3 \cdot b_{copy} \cdot \rve_{copy} + 2^2 \cdot \rve_{\texttt{UnitPropVar}} + 2^1 \cdot (1-\boldsymbol{1}[t_i=`\texttt{D}'])\cdot \rve_{\texttt{D}} + 2^0 \cdot T[(0, 0), (0,0), (1,1)] \rvr_i$$
    
\end{proof}

\begin{proposition}
    There exists a transformer with 7 layers, 5 heads, $O(p)$ embedding dimension, and $O(p^2)$ weights that, on all inputs $\vs\in \operatorname{DIMACS}(p,c)$, predicts the same token as the output as the above operations. Furthermore, let $l_{ctx}=4c+p\cdot 2^p$ be the worst-case maximum context length required to complete SAT-solving, then all weights are within $\operatorname{poly}(l_{ctx})$ and can be represented within $O(p+\log c)$ bits.
\end{proposition}

We only argue from a high level why this is true due to the complexity of the construction. In the above construction, we demonstrate how each operation can be approximated by a Self-attention or MLP layer. We can set the embedding dimension to the sum of dimensions of all the intermediate values and allocate for every intermediate values a range of dimensions that's equal to the dimension of the variables. All dimensions are initialized to 0 in the positional encoding of the transformer except for the dimensions assigned to the positional index $i$. Similarly, only the dimensions assigned to the one-hot token representation are initialized in the token embeddings. At each layer, the self-attention heads and MLP layers extract the variable values from the residual stream and perform the operations assigned to them at each layer. 

The only intermediate values whose dimensions are dependent on $p$ are the vectors for one-hot encodings and storing binary encodings of clauses and assignments. They all have size $2p$. Therefore, the number of total allocated embedding sizes is also $O(p)$.

Furthermore, \ref{sec:sacap} shows that all parameter values are polynomial with respect to the context length and the inverse of approximation errors. Note that we need only guarantee the final error is less than 1 to prevent affecting the output token. Furthermore, we can choose all parameter values so that they are multiples of $0.5$. As such, all parameters are within $\operatorname{poly}(l_{ctx})$ and can be represented by $O(\log(l_{ctx}))=O(p+\log c)$

\subsection{Correctness of Construction (\cref{thm:sat_search})}
\noindent{\it Note: This section assumes prior knowledge in propositional logic and SAT solving, including an understanding of the DPLL algorithm. For a brief explanation of the notations in this section, please refer to (\cite{abstractDPLL}). For more general knowledge, please refer to (\cite{satbook}).}

We prove that the above model autoregressive solves $\operatorname{3-SAT}_{p,c}$ by showing that it uses the CoT to simulate the  ``Abstract DPLL Procedure".

\subsubsection{Abstract DPLL}
\label{sec:appendix_abstract_dpll}

In this section, we provide a description of abstract DPLL. Since the focus of this paper is not to show the correctness of the DPLL algorithm but rather how our model's CoT is equivalent to it, we only present the main results from \cite{abstractDPLL} and refer readers to the original work for proof of the theorems.

Let $M$ be an ordered trace of variable assignments with information on whether each assignment is an \textit{decision literal} (i.e. assumption) or an \textit{unit propagation} (i.e., deduction).

For example, the ordered trace $3^{d}\: 1 \: \overline{2} \: 4^d \: 5$ denotes the following sequence of operations:

Assume $x_3=T\rightarrow$ Deduce $x_1=T\rightarrow$ Deduce $x_2=F\rightarrow$ Assume $x_4=T\rightarrow$ Deduce $x_5=T$.

Let $F$ denote a SAT formula in CNF format (which includes 3-SAT), $C$ denote a clause (e.g., $x_1 \lor \lnot x_2 \lor x_3$), $l$ denote a single literal (e.g., $\lnot x_2$), and $l^d$ denote a decision literal. Let $M \models F$ denote that the assignment in $M$ satisfies the formula $F$.

\begin{definition}[State in the DPLL Transition System]
A \emph{state} $S\in \sS$ in the DPLL transition system is either:
\begin{itemize}
    \item The special states \(\operatorname{SAT}, \operatorname{UNSAT}\), indicating that the formula satisfiable or unsatisfiable
    \item A pair \( M \parallel F \), where:
    \begin{itemize}
        \item \( F \) is a finite set of clauses \(C_1\land C_2 \dots \land C_c\) (a conjunctive normal form (CNF) formula), and
        \item \( M \) is a sequence of annotated literals \(l_1 \circ l_2 \dots \circ l_i\) for some $i\in[n]$ representing variable assignments, where $\circ$ denotes concatenation. Annotations indicate whether a literal is a decision literal (denoted by \( l^{\mathrm{d}} \)) or derived through unit propagation.
    \end{itemize}
\end{itemize}
We denote the empty sequence of literals by \( \emptyset \), unit sequences by their only literal, and the concatenation of two sequences by simple juxtaposition. While \( M \) is a sequence, it can also be viewed as a set of variable assignments by ignoring annotations and order.
\end{definition}

\begin{definition}[Adapted from Definition 1 of~\cite{abstractDPLL}]
\label{def:abstract_dpll}

The Basic DPLL system consists of the following transition rules $\sS \Longrightarrow \sS$:
\begin{flalign*}
\operatorname{UnitPropagate}:\\ M \,&\|\, F \land (C \lor l) & \quad \Longrightarrow \quad & M \circ l \,\|\, F \land (C \lor l) & \quad \textbf{if} \quad &
\begin{cases}
M \models \neg C, \\
l \text{ is undefined in } M.
\end{cases} \\
\operatorname{Decide}:\\ M \,&\|\, F & \quad \Longrightarrow \quad & M \circ l^{\mathrm{d}} \,\|\, F & \quad \textbf{if} \quad &
\begin{cases}
l \text{ or } \neg l \text{ occurs in a clause of } F, \\
l \text{ is undefined in } M.
\end{cases} \\
\operatorname{Backjump}:\\ M \circ l^{\mathrm{d}} \circ N \,&\|\, F & \quad \Longrightarrow \quad & M \circ l' \,\|\, F & \quad \textbf{if} \quad &
\begin{cases}
\text{There is some clause } C \lor l' \text{ s.t. } \\
 F \models C \lor l', \quad M \models \lnot C, \\
l' \text{ is undefined in } M, \\
l' \text{ or } \neg l' \text{ occurs in a clause of } F.
\end{cases}\\
\operatorname{Fail}:\\ M \,&\|\, F \land C & \quad \Longrightarrow \quad & \operatorname{UNSAT} & \quad \textbf{if} \quad &
\begin{cases}
M \models \neg C, \\
M \text{ contains no decision literals.}
\end{cases} \\
\operatorname{Success}:\\ M \,&\|\, F & \quad \Longrightarrow \quad & \operatorname{SAT} & \quad \textbf{if} \quad &
M \models F
\end{flalign*}
We also use $S\Longrightarrow^* S'$ to denote that there exist $S_1, S_2, \dots, S_i$ such that $S\Longrightarrow S_1 \Longrightarrow \dots \Longrightarrow S_i \Longrightarrow S'$. Also $S\Longrightarrow^! S'$ denote that $S\Longrightarrow^* S'$ and $S'$ is a final state (\op{SAT} or \op{UNSAT}).

\noindent{\it Explanation of the \(\operatorname{Backjump}\) Operation:}

The \(\operatorname{Backjump}\) operation allows the DPLL algorithm to backtrack to a previous decision and learn a new literal. In particular, $F \models C \lor l'$ means that, for some clause $C$, every assignment that satisfies $F$ must either satisfy $C$ (i.e., contain the negation of each literal in $C$) or contain $l'$ as an assignment. However, if $M\models \lnot C$, which means that $M$ conflicts with $C$ and thus contains the negation of each literal in $C$, then if we want some assignment containing $M$ to still satisfy $F$, then the assignment must also include the literal $l'$ as an assignment to ensure that it satisfies $C \lor l'$, a requirement for satisfying $F$.

In our construction, we only consider the narrower set of \op{BackTrack} operations that find the last decision and negate it:

\begin{lemma}

\label{lemma:backtrack_is_backjump}[Corrollary of Lemma 6 from \cite{abstractDPLL}]
Assume that $ \emptyset \parallel F \Longrightarrow^* M \circ l^{\mathrm{d}} \circ N \,\|\, F$, the \op{BackTrack} operation:
\begin{align*}
M \circ l^{\mathrm{d}} \circ N \,&\|\, F & \quad \Longrightarrow \quad & M \circ \lnot l \,\|\, F & \quad \textbf{if} \quad &
\begin{cases}
\text{There exists clause } C \text{ in } F \text{ such that}\\
M \circ l^{\mathrm{d}} \circ N \models \lnot C\\
N \text{contains no decision literals} \\
\end{cases}\\
\end{align*}
is always a valid \op{Backjump} operation in \cref{def:abstract_dpll}.
\end{lemma}

\end{definition}

\begin{definition}[Run of the DPLL Algorithm]
A \emph{run} of the DPLL algorithm on formula $F$ is a sequence of states \( S_0 \Longrightarrow S_1 \Longrightarrow \dots\Longrightarrow S_T \) such that:
\begin{itemize}
    \item \( S_0 \) is the initial state \( \emptyset \parallel F \)
    \item For each \( i = 0, 1, \dots, n-1 \), the transition \( S_i \Longrightarrow S_{i+1} \) is valid according to the transition rules of the DPLL system in \cref{def:abstract_dpll} (e.g., $\operatorname{UnitPropagate}$, $\operatorname{Decide}$, $\operatorname{Backjump}$, or $\operatorname{Fail}$);
    \item \( S_n \) is a final state that is either \(\operatorname{SAT}\) or \(\operatorname{UNSAT}\)
\end{itemize}
\end{definition}

Note that the above definition is simply the expansion of $\emptyset \parallel F \Longrightarrow^!S_T$.

The following theorem states that the DPLL procedure always decides the satisfiability of CNF formulas:
\begin{lemma}
\label{thm:dpll_correct} [Theorem 5 and Theorem 9 Combined from \cite{abstractDPLL}]
The Basic DPLL system provides a decision procedure for the satisfiability of CNF formulas \( F \). Specifically:
\begin{enumerate}
    \item \( \emptyset \parallel F \Longrightarrow^{!} \operatorname{UNSAT} \) if and only if \( F \) is unsatisfiable.
    \item \( \emptyset \parallel F \Longrightarrow^{!} \operatorname{SAT} \) if and only if \( F \) is satisfiable.
    \item There exist no infinite sequences of the form \( \emptyset \parallel F \Longrightarrow S_1 \Longrightarrow \cdots \)
\end{enumerate}
\end{lemma}

\subsubsection{Trace Equivalence and Inductive Proof}






To prove that our Transformer indeed simulates abstract DPLL algorithm, we use an argument of refinement: we view our Transformer construction with CoT as a state transition system and show that that transitions of this system "refines" that of the abstract DPLL state transition system:

\begin{definition}
    A transition system is a tuple $(S, T, s_0)$ where $S$ is the set of states, $T \subseteq S \times S$ is the transition relation, and $s_0$ is the start state. If $(s_1, s_2)\in T$, we say that there is a transition from $s_1$ to $s_2$ and denote $s_1 \Rightarrow s_2$.
\end{definition}

\begin{definition}
    A run $r$ of transition system $(S, T, s_0)$ is a (potentially infinite) sequence $(s_0, s_1, \dots)$ such that:
    \begin{itemize}
        \item The sequence starts with $s_0$
        \item At each step $t\geq 0$, $(s_t,s_{t+1})\in T$
    \end{itemize}
    The run $r$ \textit{halts} if it's a finite sequence such that $(s_0, s_1, \dots, s_t^*)$ such that $s_t^*$ does not have any next transitions, i.e., There's no state $s$ such that $(s_t^*,s)\in T$ 
\end{definition}

\begin{definition}[Refinement]
Given two transition systems \( A = (S_A, T_A, s_{A0}) \) and \( B = (S_B, T_B, s_{B0}) \). Transition system $A$ \emph{refines} $B$ if there is a \emph{refinement mapping} \( R \subseteq S_A \times S_B \) such that:

\begin{enumerate}
    \item $R$ maps the initial state of $A$ to the initial state of $B$:
    \[
    (s_{A0}, s_{B0}) \in R.
    \]
    \item
    For every \( (s_A, s_B) \in R \), and every run $r$ that contains $s_A$, let $r'=(s_A, \dots)$ be the suffix of $r$ starting from $s_A$. There exists $s_A'\in S_A$, $s_B'\in S_B$ such that $s_A'\in r'$ and $(s_B,s_B')\in T_B$.
    
\end{enumerate}
Here, \( \Rightarrow_A^* \) denotes the reflexive transitive closure of \( \Rightarrow_A \).
\end{definition}

\begin{proposition}
    Given two transition systems \( A = (S_A, T_A, s_{A0}) \) and \( B = (S_B, T_B, s_{B0}) \). If transition system $A$ \emph{refines} $B$, and {\it every} run of $B$ halts and ends in state $s^*_B$, then every run of $A$ contains on $s^*_A$ such that $R(s^*_A)=s^*_B$.
\end{proposition}
To proceed with this argument, we first need to define the refinement mapping between our model's CoT and the states of abstract DPLL. Consider the following model input and CoT trace:

\tok{[BOS] -2 -4 -1 0 3 4 -1 0 -1 -3 -2 0 1 -2 -4 0 -4 2 1 0 1 -2 4 0 [SEP] D 2 D 1 -4 3 [BT] D 2 D -1 -4}

Recall that \tok{[BT]} denotes backtracking and \tok{D} denotes that the next token is a decision literal.

Note that the prompt input ends at \tok{[SEP]} and the rest is the CoT produced by the model.

We want to convert this trace to a state $S=M \| F$ such that $F$ is the CNF formula in the DIAMCS encoding in the prompt input and $M$ is the "assignment trace" at the last attempt (i.e., after the last \tok{[BT]} token.). As such, $M$ correspond to the \tok{D 2 D -1 -4} portion of the trace and thus $M =2^d\: \bar{1}^d\: \bar{4}$ as described in \cref{sec:appendix_abstract_dpll}. We formalize this process as follows:

\begin{definition}[Translating CoT to Abstract DPLL State]
\label{def:cot_to_dpll}

For any number of variables $p\in \sN^+$, let $\gV$ be the set of tokens:
\[
\gV = \{\, \tok{-i},\ \tok{i} \mid i \in [p] \,\} \cup \{\, \tok{D},\ \tok{[SEP]},\ \tok{[BOS]},\ \tok{[BT]},\ \tok{0},\ \tok{SAT},\ \tok{UNSAT} \,\}.
\]

Define a mapping $f_{\gS}: \gV^* \rightarrow \gS \cup \{\operatorname{error}\}$ that converts a sequence of tokens $R \in \gV^*$ into an abstract DPLL state as follows:

\begin{enumerate}
    \item \textbf{If} $R$ ends with $\tok{SAT}$ or $\tok{UNSAT}$, \textbf{then} set $M_{\gS}(R)$ to $\operatorname{SAT}$ or $\operatorname{UNSAT}$ accordingly. 

    \item \textbf{Else if} $R$ contains exactly one $\tok{[SEP]}$ token, split $R$ at $\tok{[SEP]}$ into $R_{\text{DIMACS}}$ and $R_{\text{Trace}}$.

    \item Parse $R_{\text{DIMACS}}$ as a DIMACS representation of CNF formula $F$, assuming it starts with $\tok{[BOS]}$ and ends with $\tok{0}$. If parsing fails, set $M_{\gS}(R) = \operatorname{fail}$.

    \item Find the last \tok{[BT]} in $R_{\text{Trace}}$, and let $R_{\text{current}}$ be the part of $R_{\text{Trace}}$ after the last \tok{[BT]}. If there's none, set $R_{\text{current}}$ to $R_{\text{Trace}}$.

    \item Initialize an empty sequence $M$ to represent variable assignments and set a flag $\textit{isDecision} \leftarrow \texttt{False}$.

    \item Process each token $t$ in $R_{\text{current}}$ sequentially:

    \begin{itemize}
        \item \textbf{If} $t = \tok{D}$, set $\textit{isDecision} \leftarrow \texttt{True}$.

        \item \textbf{Else if} $l$ is a literal, append $l$ to $M$, annotated as a decision literal if $\textit{isDecision} = \texttt{True}$, or as a unit propagation otherwise.

        \item Reset $\textit{isDecision} \leftarrow \texttt{False}$.

        \item \textbf{Else}, set $M_{\gS}(R) = \operatorname{error}$.

    \end{itemize}

    \item \textbf{Return} the state $M \parallel F$.

\end{enumerate}

\end{definition}

With the above mapping, we can specify the following properties of our Transformer construction based on logical relations between $A$ and $F$:

\begin{proposition}
    Given input sequence $s_{1:n} \in \gV^*$ such that $f_\gS(s_{1:n})=M\;\|\;F$ for which $F$ is a valid 3-SAT formula and $M$ is a sequence of annotated literals. Let $A$ be the partial assignment corresponding to $M$ (i.e., removing annotation and order). Let $D:=\{l \in L\;|\;F\land A\models_1 l\}$ be the set of literals that can be deduced through unit propagation. Let $U$ be the set of literals corresponding to variables not assigned in $A$. Let $s_{n+1}$ be the output of the Transformer model defined in \cref{sec:construction} when given $s_{1:n}$ as input, then $s_{n+1}$ satisfy the following:
    \begin{align*}
        A\models F & \Longrightarrow s_{n+1}=\cd{SAT}\\
        (M\;\textnormal{contains no decision literals})\;\land\;(F\models\lnot A) & \Longrightarrow s_{n+1}=\cd{UNSAT}\\
        (M\;\textnormal{contains decision literals})\;\land\;(F\models\lnot A) & \Longrightarrow s_{n+1}=\cd{[BackTrack]}\\
        (A\not \models F)\land (F\not \models \lnot A)\land (D\neq \emptyset)&\Longrightarrow s_{n+1}\in D\\
        (A\not \models F)\land (F\not \models \lnot A)\land (D= \emptyset)\land(s_n\neq \tok{D})&\Longrightarrow s_{n+1}=\tok{D}\\
        (A\not \models F)\land (F\not \models \lnot A)\land (D= \emptyset)\land(s_n= \tok{D})&\Longrightarrow s_{n+1}\in U
    \end{align*}
\end{proposition}

We now present the inductive lemma:

\begin{lemma}[Inductive Lemma]
\label{lemmainductive}
    For any $p, c\in \sN^+$, for any input $F_{\text{DIMACS}}\in \operatorname{DIMACS}(p,c)$ of length $n$, let $F$ be the boolean formula in CNF form encoded in $F_{\text{DIMACS}}$. Let $A$ be the model described in section \ref{sec:construction} with parameters $p, c$. Let $(\vs_{1:n}, \vs_{1:n+1}, \dots)$ be the trace of $\vs$ when running the Greedy Decoding \cref{alg:next-token-prediction} with model $A$ and input prompt $\vs_{1:n}=F_{\text{DIMACS}}$. 
    For every $i\in \sN^+$, if  $f_\gS(\vs_{1:n+i})=S$ and $S\notin \{\operatorname{SAT}, \operatorname{UNSAT}, \operatorname{error}\}$, then there exist $j\in \sN^+$ and $S'\in \sS$ such that $S\Longrightarrow S'$ and $f_{\gS}(\vs_{1:n+i+j})=S'$.
\end{lemma}

We now show trace equivalence between the model $A$ and some instantiating of the abstract DPLL with a specific heuristic:

\begin{definition}
    For any heuristic $h: \sS \rightarrow \gL$ where $\gL$ is the set of literals, let \op{DPLL}$_h$ denote an instantiation of the abstract DPLL algorithm that selects $h(S)$ as the decision literal when performing $\operatorname{Decide}$ and only performs the \op{BackTrack} operation for \op{Backjump}. $h(S)$ is a valid heuristic if \op{DPLL}$_h$ always abides by the \op{Decide} transition.
\end{definition}

\begin{lemma}(Trace Simulation)
    There exists a valid heuristic $h: \sS \rightarrow \gL$ for which the Transformer model $A$ is trace equivalent to $\operatorname{DPLL}_h$ on all inputs in $\operatorname{DIMACS}(p,c)$
\end{lemma}
\begin{proof}
We aim to show that there exists a valid heuristic $h: \mathcal{S} \rightarrow \mathcal{L}$ such that the Transformer model $A$ is trace equivalent to $\operatorname{DPLL}_h$ on all inputs in $\operatorname{DIMACS}(p,c)$.

Define the heuristic $h$ as follows: For any state $S \in \mathcal{S}$, let $h(S)$ be the literal that the Transformer model $A$ selects as its next decision literal when in state $S$.

Formally, given that the model $A$ outputs tokens corresponding to decisions, unit propagations, backtracks, etc., and that these tokens can be mapped to transitions in the abstract DPLL system via the mapping $M_{\mathcal{S}}$ (as per the \textit{Translating CoT to Abstract DPLL State} definition), we set:
\[
h(S) = \begin{cases}
\text{the decision literal chosen by } A \text{ in state } S, & \text{if } A \text{ performs a } \operatorname{Decide} \text{ transition}, \\
\text{undefined}, & \text{otherwise}.
\end{cases}
\]
This heuristic is valid because $A$ always abides by the $\operatorname{Decide}$ transition rules, ensuring $h(S)$ selects a literal that occurs in $F$ and is undefined in $M$, satisfying the conditions of a valid heuristic.

Define a mapping $\phi: \Sigma_A \rightarrow \Sigma_B$, where $\Sigma_A$ is the set of possible states of model $A$, and $\Sigma_B$ is the set of possible states of $\operatorname{DPLL}_h$, such that for any state $S$ in the execution trace of $A$, $\phi(S) = S$. That is, we identify the states of $A$ with the corresponding states in $\operatorname{DPLL}_h$ by mapping the sequence of assignments and the formula $F$ directly.

\textbf{Proof of Refinement:}

We proceed by induction on the number of steps in the execution trace.

\emph{Base Case ($i = 0$):}

At the beginning, both algorithms start from the initial state with no assignments:
\[
\text{For } A: \quad S_0^A = \emptyset \parallel F, \quad \text{and} \quad \text{For } \operatorname{DPLL}_h: \quad S_0^B = \emptyset \parallel F.
\]
Clearly, $\phi(S_0^A) = S_0^B$.

\emph{Inductive Step:}

Assume that after $k$ steps, the states correspond via $\phi$:
\[
\phi(S_k^A) = S_k^B.
\]
We need to show that after the next transition, the states still correspond, i.e., $\phi(S_{k+1}^A) = S_{k+1}^B$.

Suppose the model $A$ applies a $\operatorname{UnitPropagate}$ operation, transitioning from state $S_k^A$ to $S_{k+1}^A$ by adding a literal $l$ deduced via unit propagation.

Since unit propagation is deterministic and depends solely on the current assignment $M$ and formula $F$, $\operatorname{DPLL}_h$ will also apply the same $\operatorname{UnitPropagate}$ operation, transitioning from $S_k^B$ to $S_{k+1}^B$ by adding the same literal $l$.

Thus, $\phi(S_{k+1}^A) = S_{k+1}^B$.

Suppose the model $A$ applies a $\operatorname{Decide}$ operation, transitioning from $S_k^A$ to $S_{k+1}^A$ by adding a decision literal $l = h(S_k^A)$.

By the definition of the heuristic $h$, $\operatorname{DPLL}_h$ also selects $l$ as the decision literal in state $S_k^B$. Both algorithms make the same decision and transition to the same next state.

Therefore, $\phi(S_{k+1}^A) = S_{k+1}^B$.

Suppose the model $A$ applies a $\operatorname{Backjump}$ operation, backtracking to a previous state and assigning a new literal.

Since $\operatorname{DPLL}_h$ performs only the $\operatorname{BackTrack}$ operation for $\operatorname{Backjump}$ (as per the definition), and $A$ simulates this operation, both algorithms backtrack in the same manner and update their assignments accordingly.

Thus, $\phi(S_{k+1}^A) = S_{k+1}^B$.

If the model $A$ reaches a terminal state indicating $\operatorname{SAT}$ or $\operatorname{UNSAT}$, then so does $\operatorname{DPLL}_h$, since their sequences of transitions have been identical up to this point.

In all cases, the next state of model $A$ corresponds to the next state of $\operatorname{DPLL}_h$ under the mapping $\phi$. Therefore, by induction, the execution traces of $A$ and $\operatorname{DPLL}_h$ are such that for all $i$,
\[
\phi(S_i^A) = S_i^B.
\]

Since the heuristic $h$ selects the same decision literals as the model $A$, and $A$ always abides by the $\operatorname{Decide}$ transition (as per its design), $h$ is a valid heuristic according to the definition provided.

\end{proof}

\section{\cname and Compiled Theoretical Construction}
\label{sec:appendix_compiler}

\subsection{Supported Features and Operations}
\label{sec:compiler_features}

Our tool is designed to provide an intuitive syntax resembling standard numerical array manipulation, akin to NumPy, while supporting a diverse and extensible set of abstract operations. \cname is capable of implementing

\begin{itemize}
\item \textbf{NumPy-like Array Syntax} for indexing, arithmetic, and comparison.
\item \textbf{Multi-Level Abstraction} to enable low-level customization.
\item \textbf{Multi-stage Evaluation Mechanisms} to facilitate debugging and error localization
\item \textbf{High Extensibility} through structured class inheritance, promoting the addition of new features and operations.
\end{itemize}

Each intermediate ``variable" is an instance of the \cd{SOp} base class (name adapted from \cite{lindner_tracr_2023}), and each instance \cd{sop} of \cd{SOp} is assigned a dimension $d_\cd{sop}\in \sN^{+}$ and can be viewed as an abstract representation of an $\mathbf{\sR}^{n\times d_{\cd{sop}}}$ array where $n$ is the number of tokens in the input to the Transformer model. A \cname ``program" is basically a sequence of array operations over \cd{SOp}s. 

Throughout this section, we refer to the indices along the first dimension of an \cd{SOp} as ``position" and refer to indices along the second dimension as ``dimension".

The ``inputs" to a program are arbitrary positional encoding and token embedding variables, represented by the base class names \cd{PosEncSOp} and \cd{TokEmbSOp} respectively. For example, the \cd{OneHotTokEmb} class represents the one-hot embedding of tokens and \cd{Indices} represents the numerical value of the index of each position.

The rest of the program performs various operations that compute new \cd{SOps} based on existing ones. We provide implementations of basic building block operations including (but not limited to) the following:
\begin{itemize}
    \item \cd{Mean(q, k, v)} Represents the “Averaging Hard Attention” operation. At each position $i$, this operation identifies all target positions $j$ with the maximal value of $q_i^\top k_j$ for $j \leq i$ and computes the average of the corresponding $v_j$ values.
    \item \cd{sop[idx, :]} Performs indexing using a one-dimensional index array \cd{idx}, producing an \cd{SOp} \cd{out} such that $\cd{out}[i, j] = \cd{sop}[\cd{idx}[i], j]$ for $i \in [n]$ and $j \in [d_{\cd{sop}}]$. This mirrors NumPy’s array indexing semantics.
    \item \cd{sop[:, start:end]} Extracts a slice of dimensions from \cd{sop}, where \cd{start}, \cd{end} $\in [d_{\cd{sop}}]$, resulting in a new \cd{SOp} of dimension $\cd{end} - \cd{start}$. This operation is analogous to NumPy slicing.
    \item Element-wise operations such as \cd{sop1 + sop2}, \cd{sop1 - sop2}, \cd{sop1 * sop2}, logical operations (\cd{\&} for AND, \cd{|} for OR), and comparison operations ($\geq$, $\leq$, $>$, $<$), following standard broadcasting rules.
\end{itemize}

As an illustrative example, the following function returns a one-dimensional \cd{SOp} representing the position index of the closest token within a set of target tokens:

\begin{minted}[frame=lines, bgcolor=bg, fontsize=\small, linenos]{python}
def nearest_token_id(tok_emb: OneHotTokEmb, vocab: List[str],
                     targets: List[str], indices: Indices=indices):
    # Get the token ids of the target tokens
    target_tok_ids = [vocab.index(target) for target in targets]
    # Get whether the current token is one of the target tokens
    # by summing the one-hot embedding
    target_token_embs = Concat([tok_emb[:, target_tok_id]
                                for target_tok_id in target_tok_ids])
    in_targets = target_token_embs.sum(axis=1)
    # Filter the indices to only include the target tokens
    filtered_index = indices * in_targets
    return filtered_index.max()
\end{minted}

We present  our full code implementing our construction for \cref{thm:sat_search} using \cname in \cref{sec:appendix_code}.

\subsection{Comparison with Tracr~\citep{lindner_tracr_2023}}
\label{sec:vs_tracr}

While Tracr also compiles RASP programs into Transformer weights, the RASP language is designed to provide a concise description of the class of functions that Transformers can easily learn. As such, RASP has minimal syntax and is designed to represent relatively simple sequence operations such as counting, sorting, etc. In contrast, our tool is designed to help construct theoretical constructions that implement relatively more complex algorithms.

In our preliminary attempt to implement our SAT solver model with Tracr, we identified several implementation inconveniences and limitations of Tracr when scaling to more complex algorithms, which motivated the development of our tool. In particular: 
\begin{itemize}
    \item Every ``variable" (termed \cd{sop} in \cite{lindner_tracr_2023}) in Tracr must be either a one-hot categorical encoding or a single numerical value. This constraint makes representing more complex vector structures highly inconvenient. Furthermore, each \cd{select} operation (i.e., self-attention) accepts only a single \cd{sop} as the query and key vectors, whereas our theoretical construction often requires incorporating multiple variables as queries and keys.

    In contrast, each variable in \cname represents a 2-D array, which facilitates the implementation of vector-based operations such as performing logical deductions as described in \cref{lemma:vec_ded}
    
    \item In terms of parameter complexity, Tracr represents position indices and many other discrete \cd{sop}s with a one-hot encoding, allocating a residual stream dimension for each possible value of the \cd{sop}. In particular, compiling models with a context length of $n$ requires $O(n)$ additional embedding dimensions for each SOp that represents a position index. For each binary operation between one-hot encoded \cd{sop}s (such as position indices), Tracr creates an MLP layer that first creates a lookup table of all possible value combinations of the input \cd{sop}s. This results in an MLP layer of $O(n^3)$ parameters.

    In contrast, our tool directly represents numerical values rather than working with token representations. For example, positional encodings only take up 1 dimension of the residual stream, which drastically reduces the number of parameters for longer context lengths.
    
\end{itemize}

We would like to emphasize that our goal is not to replace Tracr or RASP, which have unparalleled simplicity and interpretability in describing well-studied sequence operations. The goal of our tool is to assist with creating implementations of theoretical constructions to help verify its behaviors and investigate internal properties.

\subsection{The Compilation Process}
\cname takes in an \cd{out} variable that contains the computational graph of the algorithm and outputs a PyTorch (\cite{pytorch}) model. The compilation process follows stages similar to those of Tracr:
\begin{enumerate}
    \item \textbf{Computational Graph Construction}: When a user writes \cd{sop} operations, each operation automatically creates a dependency tree of all operations required for computing the resulting \cd{sop} value.
    \item \textbf{Reduction to Base Operations:} Each \cd{sop} operation is reduced to one of 5 base classes: \cd{SelfAttention} for operation that requires information from other token positions, \cd{GLUMLP} for non-linear local operations, \cd{Linear} for linear local operations, \cd{PosEncSOp} for positional encodings, or \cd{TokEmbSOp} for token embeddings. Sequential \cd{Linear} operations are reduced to a single operation through matrix multiplication and dependency merging.
    \item \textbf{Allocation of Layers and Residual Stream:} The computational graph is topologically sorted such that each \cd{sop} appears later than its dependencies. This sorting is then used to assign \cd{SelfAttention} and \cd{GLUMLP} \cd{sop}s to Transformer layer numbers that comply with dependency constraints. Furthermore, each non-\cd{Linear sop} is also allocated a portion of the residual stream equal to their $d_{\cd{sop}}$ size.
    \item \textbf{Model Construction and Weight Assignment:} A PyTorch model is initialized based on the number of required layers, hidden size, and embedding size inferred from the previous steps. The computed weights for each \cd{sop} are assigned to different model components based on their types. Notably, each \cd{SelfAttention sop} corresponds to an attention head, and each \cd{GLUMLP} sops corresponds to part of a MLP layer with \cd{ReGLU} activation.
\end{enumerate}

\noindent{\bf Soft vs Hard Attention} The reduction of \cd{Mean} to \cd{SelfAttention} induces inevitable numerical errors due to \cd{Mean} representing averaging a strict subset of previous positions while \cd{SelfAttention} computes a weighted average over all previous positions via $\operatorname{softmax}$. This error also affects other operations based on \cd{Mean} such as position indexing. We control this error via an ``exactness" parameter $\beta$ that scales the attention logits, and \cref{lemma:TM_mean} shows that the error decreases exponentially w.r.t. $\beta$. 

\noindent{\bf Multi-Stage Evaluation} To facilitate debugging, \cname allows 3 types of evaluations for every \cd{sop} at different stages of compilation.
\begin{itemize}
    \item \cd{sop.abstract\_eval(tokens)} evaluates \cd{sop} on a sequence of input tokens without any numerical errors. This can be used to validate the correctness of the algorithm implementation as \cd{sop} operations.
    \item \cd{sop.concrete\_eval(tokens)} evaluates \cd{sop} on an input sequence after reducing to the base classes at step 2 of the compilation process. This helps localize errors stemming from incorrect reduction of high-level operations to base classes.
    \item \textbf{Model evaluation} This corresponds to evaluating the Pytorch model after the full compilation process.
\end{itemize}

\subsection{Code for Theoretical Construction}

The following code is used to construct the Transformer specification passed as input to \cname. To facilitate easier implementation, we interleave PARAT statements with Python and Numpy operations when appropriate. PARAT takes the return variable \cd{out} as input and produces the theoretical construction discussed in \cref{sec:compiled_analysis}

\label{sec:appendix_code}
\begin{minted}[frame=lines, bgcolor=bg, fontsize=\small, linenos]{python}

def nearest_token(tok_emb: OneHotTokEmb, vocab: List[str],
                  targets: List[str], v: SOp | List[SOp],
                  indices: PosEncSOp = indices):
    if not isinstance(v, list):
        v = [v]

    target_tok_ids = [vocab.index(target) for target in targets]
    target_tokens = Concat([tok_emb[:, target_tok_id] 
        for target_tok_id in target_tok_ids])
    in_targets = Linear(target_tokens, np.ones((1, len(targets))))
    filtered_index = (indices * in_targets)

    new_v = []
    for v_i in v:
        if isinstance(v_i, SOp):
            new_v.append(v_i)
        elif v_i == 'target' or v_i == 'targets':
            new_v.append(target_tokens)
        else:
            raise ValueError('Unsupported value type')

    return Mean(ones, filtered_index, new_v, bos_weight=1)


def t(encodings: SOp, num_vars,
      true_vec=(1, 0),
      false_vec=(0, 1),
      none_vec=(0, 0),
      ones: Ones = ones):
    mat = np.zeros((2 * num_vars, 2 * num_vars))
    true_vec_off = (true_vec[0] - none_vec[0], true_vec[1] - none_vec[1])
    false_vec_off = (false_vec[0] - none_vec[0], false_vec[1] - none_vec[1])
    for i in range(num_vars):
        true_id = i
        false_id = num_vars + i
        mat[true_id, true_id] = true_vec_off[0]
        mat[true_id, false_id] = false_vec_off[0]
        mat[false_id, true_id] = true_vec_off[1]
        mat[false_id, false_id] = false_vec_off[1]

    bias = np.zeros(2 * num_vars)
    bias[:num_vars] += none_vec[0]
    bias[num_vars:] = none_vec[1]

    return Linear([encodings, ones], 
        np.hstack([mat.T, bias.reshape((-1, 1))]))


def dpll(num_vars, num_clauses, context_len,
         mean_exactness=20, nonsep_penalty=20,
         return_logs=False) -> Tuple[
    SOp, List, Dict[str, SOp]]:
    vocab: List = ([str(i) for i in range(1, num_vars + 1)]
                   + [str(-i) for i in range(1, num_vars + 1)]
                   + ['0', '[SEP]', '[BT]', '[BOS]', 'D', 'SAT', 'UNSAT'])
    idx: Dict[str, int] = {token: idx for idx, token in enumerate(vocab)}
    sop_logs: Dict[str, SOp] = {}
    sops.config["mean_exactness"] = mean_exactness
    # Initialize Base SOps
    tok_emb = OneHotTokEmb(idx).named("tok_emb")

    nearest_sep = nearest_token(tok_emb=tok_emb,
                                vocab=vocab,
                                targets=['0', '[SEP]', '[BT]'],
                                v=[indices, 'target']).named(
        "nearest_sep")

    # The nearest (including self) separator token and whether
    # the previous separator token is '0', '[SEP]', '[UP]', '[BT]'
    p_i_sep_p, b_0, b_SEP, b_BackTrack = (
        nearest_sep[:, 0].named("p_i_sep_p"),
        nearest_sep[:, 1].named("b_0"),
        nearest_sep[:, 2].named("b_SEP"),
        nearest_sep[:, 3].named("b_BackTrack"))

    # The nearest 'D' token, which denotes the next token is a decision
    p_i_D = nearest_token(tok_emb=tok_emb, vocab=vocab, targets=['D'],
                          v=indices).named("p_i_D")

    prev_pos = Id([p_i_sep_p, tok_emb[:, idx['D']]])[indices - 1]
    # p_i_sep: The previous (excluding self) separator token
    p_i_sep = (prev_pos[:, 0] - is_bos).named("p_i_sep")

    # b_decision: whether the current position is a decision literal
    b_decision = prev_pos[:, 1].named("b_decision")

    # The distance to the nearest separator, 
    # i.e., the length of the current state
    d_i_sep = (indices - p_i_sep_p).named("d_i_sep")

    # Attention operation for representing the current
    # clause/assignment as a bitvector of dimension 2d
    p_i_sep_2 = (p_i_sep * p_i_sep).named("p_i_sep_2")
    e_vars = tok_emb[:, : 2 * num_vars].named("e_vars")
    r_i_pre = Mean(q_sops=[p_i_sep_2, p_i_sep, ones],
                   k_sops=[-ones, 2 * p_i_sep, -p_i_sep_2],
                   v_sops=e_vars).named("r_i_pre")
    r_i = (r_i_pre * (indices - p_i_sep)).named("r_i")

    # The position of the previous (excluding self) separator token
    p_i_sep_min = p_i_sep[p_i_sep_p].named("p_i_sep_min")

    # The same position in the previous state.
    # This is used for copying from the previous state
    p_i_min = (p_i_sep_min + d_i_sep + num_vars * b_SEP).named("p_i_min")

    # The position of the last decision in the previous state
    p_i_D_min = p_i_D[p_i_sep_p].named("p_i_D_min")

    # Is the next token the literal resulting from backtracking?
    b_D_min = (p_i_D_min == p_i_min + 1).named("b_D_min")

    # Check if the current assignment satisfies the formula
    # (See Theorem Proof for justification)
    sat_q = [r_i, ones]
    sat_k = [-r_i, (-nonsep_penalty) * (1 - tok_emb[:, idx['0']])]
    sat_v = is_bos
    b_sat = (Mean(sat_q, sat_k, sat_v,
        bos_weight=nonsep_penalty - 0.5) > 0).named("b_sat")

    # Check if the current assignment contracdicts the formula
    # (See Theorem Proof for justification)
    unsat_q = [t(r_i, num_vars, true_vec=(1, 0), 
        false_vec=(0, 1), none_vec=(1, 1)), ones]
    unsat_k = sat_k
    unsat_v = 1 - is_bos
    b_cont = (Mean(unsat_q, unsat_k, unsat_v, 
        bos_weight=nonsep_penalty - 0.5) > 0).named("b_cont")
    b_copy_p = (p_i_min < (p_i_sep_p - 1)).named("b_copy_p")



    # Unit Propagation
    up_q = unsat_q
    up_k = unsat_k
    up_v = num_clauses * r_i
    o_up = Mean(up_q, up_k, up_v, bos_weight=nonsep_penalty - 1.5)


    e_up = (
            GLUMLP(act_sops=(o_up - t(r_i, num_vars,
                                      true_vec=(1, 1),
                                      false_vec=(1, 1),
                                      none_vec=(0, 0))))
            - GLUMLP(act_sops=(o_up - 1))
    ).named("e_up_new")


    # Heuristic for decision literal selection: 
    # Find the most common literal in remaining clauses
    heuristic_q = [t(r_i, num_vars, true_vec=(-10, 1),
                     false_vec=(1, -10), none_vec=(0, 0)), ones]
    heuristic_k = [r_i, (-nonsep_penalty) * (1 - tok_emb[:, idx['0']])]
    heuristic_v = r_i
    heuristic_o = SelfAttention(heuristic_q, heuristic_k, heuristic_v)

    # Whether the current assignment contains no decision literal
    b_no_decision = (p_i_D <= p_i_sep).named("b_no_decision")

    # Whether Backtracking is finished
    b_BT_finish = ((p_i_D_min <= p_i_min) & b_BackTrack)

    # The negation of the last decision literal in the previous state
    e_BT = t(e_vars[p_i_D_min + 1], num_vars=num_vars,
             true_vec=(0, 1), false_vec=(1, 0), none_vec=(0, 0))

    # The next index in the previous state for copying
    p_i_min_index = (p_i_min + 1).named("p_i_min_index")

    # The next token in the previous state for copying
    e_copy = tok_emb[p_i_min_index].named("e_copy")

    # Whether we've decided that the formula is UNSAT
    b_unsat = (b_no_decision & b_cont).named("b_unsat")

    # Whether we're negativing the last decision literal for backtracking
    b_backtrack = (b_D_min & b_BackTrack).named("b_backtrack")

    # Whether we're copying tokens from the previous state
    b_copy = (b_copy_p & (1 - b_BT_finish)).named("b_copy")

    b_BT_token = (b_cont & (1 - tok_emb[:, idx['[BT]']]))
    b_not_D = (1 - tok_emb[:, idx['D']]).named("b_not_D")
    e_unassigned = t(r_i, num_vars, true_vec=(0, 0), 
        false_vec=(0, 0), none_vec=(1, 1)).named("e_unassigned")

    out = CPOutput(len(vocab), 
        [(b_sat, idx['SAT'], 16),
        (b_unsat, idx['UNSAT'], 15),
        (b_BT_token, idx['[BT]'], 14),
        (b_backtrack, Pad(e_BT, len(vocab), idx['1']), 12),
        (b_copy, e_copy, 6),
        (None, Pad(e_up, len(vocab), idx['1']), 4),
        (b_not_D, idx['D'], 3),
        (None, Pad(e_unassigned + heuristic_o,
                   out_dim=len(vocab), start_dim=idx['1']), 1)])

        return out
\end{minted}

\end{document}